\documentclass[10pt]{article}

\usepackage{amsmath,amsfonts,amsthm,commands,graphicx,paper}

\usepackage{graphicx}
\usepackage{float}
\usepackage{algorithm}
\usepackage[noend]{algpseudocode}
\usepackage{longtable}
\usepackage{booktabs}
\usepackage{parskip}
\usepackage{longtable}
\usepackage[
  separate-uncertainty = true,
  multi-part-units = repeat
]{siunitx}
\usepackage{esvect}
\usepackage{multirow}
\usepackage{mathtools}
\usepackage{algorithm}% http://ctan.org/pkg/algorithm
\usepackage{algpseudocode}% http://ctan.org/pkg/algorithmicx

\usepackage{hyperref}
\usepackage{amssymb}
\usepackage[justification=centering]{caption}
\usepackage{multirow}
\usepackage{caption} 
\usepackage{bbm}
\usepackage{subcaption,float}
\usepackage{graphicx} % more modern
\newcommand{\inclu}[0] {\ar@{^{(}->}}

\newcommand{\bbmatrix}     {\begin{bmatrix}}
\newcommand{\ebmatrix}     {\end{bmatrix}}

\algdef{SE}[DOWHILE]{begi}{en}{\textbf{begin}}[1]{\algorithmicend\ #1}%
\algdef{SE}{Pre}{EndPre}[1]{\textbf{preloop} \(\mbox{#1}\) \textbf{do}}{\textbf{end}}%
\algdef{SE}[SUBALG]{Indent}{EndIndent}{}{\algorithmicend}%
\algtext*{Indent}
\algtext*{EndIndent}

\title{Novel and Efficient Approximations for Zero-One Loss of Linear Classifiers}

\author{Hiva Ghanbari
        \thanks{Department of Industrial and Systems Engineering, Lehigh University, Bethlehem, PA, USA.  E-mail: {\tt hiva.ghanbari@gmail.com}}
        \and
        Minhan Li
        \thanks{Department of Industrial and Systems Engineering, Lehigh University, Bethlehem, PA, USA.  E-mail: {\tt mil417@lehigh.edu}}
        \and
        Katya Scheinberg
        \thanks{Department of Industrial and Systems Engineering, Lehigh University, Bethlehem, PA, USA.  E-mail: {\tt katyascheinberg@gmail.com}}
        }
\date{}

\begin{document}

% Make title
\maketitle

%*********
% Section
%*********

%*********
% Section
%*********
\begin{abstract} 
\noindent The predictive quality of machine learning models is typically measured in terms of their (approximate) expected prediction accuracy or the  so-called Area Under the Curve (AUC). Minimizing the reciprocals of these measures--the expected risk and ranking loss--are  the goals of supervised learning. However, when the models are constructed by the means of  empirical risk minimization (ERM), surrogate functions such as the logistic loss or hinge loss are optimized instead. This is done  because empirical approximations  of the expected  error and the ranking loss are step functions that have zero derivatives almost everywhere. In this work, we show that in the case of linear predictors,  the expected error and the expected ranking loss can be effectively approximated by smooth functions whose closed form expressions and those of their first (and second) order derivatives depend on the first and second moments of the data distribution, which can be precomputed. 
Hence,  the complexity of an optimization algorithm applied to these functions does not depend on
the size of the training data. These approximation functions are derived under the assumption that the output of the linear classifier for a given data set has an approximately normal distribution. We argue that this assumption is significantly weaker than the Gaussian assumption on the data itself and we support this claim by demonstrating that our new approximation is quite accurate on data sets that are not necessarily Gaussian. We present computational results that show that our proposed approximations and related optimization algorithms can produce linear classifiers with similar or better test accuracy or AUC,  than those obtained using state-of-the-art approaches, in a fraction of the time. 

\end{abstract} 

%**********
% Sections
%**********
%Introduction
\section{Introduction} \label{sec.introduction}
In this paper, we propose a  novel and efficient approach to empirical risk minimization for binary linear classification. Our main motivation is that the true goal of a machine learning algorithm is  to minimize the expected prediction error of a classifier,  or, in the case of imbalanced data sets, the expected ranking loss  \cite{hanley}.
 However, since the standard sample average approximation of the expected prediction error is a discontinuous step function, whose gradient is either zero or not defined, other continuous and convex loss functions,  such as hinge loss and logistic loss, whose sample average approximations have useful derivatives, are optimized instead.
Here we argue that, unlike their sample average approximations,  the expected error and ranking loss functions of a linear classifier  are themselves often smooth and thus can be efficiently optimized via gradient-based optimization methods, if only  accurate gradient estimates can be obtained. We thus derive novel direct approximation functions of  the expected error and ranking loss, which can be written in closed form, using first and second moments of the data distributions. We then 
apply further approximation by replacing the exact moments with empirical moments and thus obtain functions that can be efficiently optimized by using first or second order methods. The resulting functions are not convex, but our empirical results show that good solutions are obtainable by standard optimization methods.

Our approximation of the expected error and ranking loss functions are derived under the assumption that, for any useful linear classifier $w$,  its  output  $w^\top x$, over a given data set, is approximately normal with appropriate first and second order moments, 
which are  functions of $w$ and the first and second moments of the data distribution. We will argue that this assumption is a lot milder than the assumption that the data itself is nearly Gaussian. Then under the assumption that $w^\top x$ {\em exactly}  obeys normal distribution, with given moments, we derive the closed form of the expected error and ranking loss functions. We then use these functions and their derivatives to approximately optimize the  empirical error and ranking loss for a given training data set. The only dependence of these functions on the training data is via the first and second moments which can be computed prior to applying optimization. Thus the complexity of each iteration of the optimization algorithm is independent of the data set size. This is in contrast with other standard ERM methods, such as logistic regression and support vector machines. 
This distinction is of particular importance in optimizing the  ranking  loss because replacing it with surrogate loss functions such as  pairwise logistic loss or pairwise hinge loss whose gradient computation has superlinear dependence on the data set size.

Through empirical experiments we show that  optimizing the derived functional forms of prediction error and ranking loss, using empirical approximate moments, produces competitive (and sometimes better) predictors  with  those obtained by optimizing surrogate approximations, such as logistic and hinge losses. This behavior is in contrast with, for example, Linear Discriminant Analysis (LDA) \cite{LDA},  which is the method to compute linear classifiers under the Gaussian assumption.

 In the next section we introduce some definitions and preliminaries. In Section \ref{sec.ER.AUC}  we derive our proposed approximations and  provide some theoretical justifications. We present computational results and plots demonstrating the quality of the approximations in Section \ref{sec.numerical}, and finally, we state our conclusions in Section \ref{sec.conclusion}.

 %\cite{cortes2,Osuna,Smola} 
 % \cite{Hosmer}, 

%*********
% Section
%*********
\section{Preliminaries and Problem Description}\label{sec.prelim}
We consider the classical setting of  \emph{supervised} machine learning, where we are given a finite \emph{training set} $\mathcal{S}$ of $n$ pairs,
\begin{equation*} \label{training.set}
\mathcal{S} := \{(x_i,y_i)~:~ i =1,\cdots,n\},
\end{equation*}
where $x_i \in \mathbb{R}^d$ are the \emph{input}  vectors of \emph{features} and $y_i \in\{+1,-1\}$ are the \emph{binary output} labels. It is assumed that each pair $(x_i,y_i)$ is an i.i.d.  sample of the random variable $(X,Y)$ with some unknown joint probability distribution.
% In particular, we are interested in \emph{binary classification} where the target variables $Y$ can only take two discrete values $\{+1,-1\}$. 
%representing the probability distribution $P_X(x)$ as well as the conditional probability $P_{Y|X}(y|x)$ for a fixed input $x$.
%The set $\mathcal{S}$ is known as a \emph{training set}. 
Let  $f(x;w) $ denote a \emph{classifier function}, parametrized by $w$. 
As discussed above, there are two different performance measures by which one may evaluate the quality of  $f$: the expected prediction accuracy,  and the expected AUC. A reciprocal of the expected accuracy is the expected error, which 
 is defined as follows
\begin{equation} \label{expectedRisk}
\begin{aligned}
F_{01}\left( w \right) =  P({Y}\cdot  f(X;w) <0)= \mathbb{E}_{{X},{Y}} \left[  \mathbbm{1} \left({Y}  \cdot f(X;w) \right)\right ] 
%\int_{\mathcal{X}} \int_{\mathcal{Y}} P_{X,Y}(x,y) \ell_{01} \left(f(x;w),y\right) dy dx,
\end{aligned}
\end{equation}
where
\begin{equation} \label{step_0,1}
 \mathbbm{1}\left(v\right) =  \begin{cases} 1 & \text {if} ~~ v <0,\\ 0 & \text {if} ~~ v \geq 0. \end{cases}
\end{equation}

A  sample average approximation of \eqref{expectedRisk}, on $\mathcal{S}$,  is 
\begin{equation} \label{func_0,1}
\hat F_{01}\left(w;\mathcal{S}\right) = \frac{1}{n}\sum_{i=1}^{n}  \mathbbm{1}\left (y_i \cdot f(x_i;w)\right ).
\end{equation}

The difficulty of optimizing  \eqref{func_0,1}, even approximately, arises from the fact that its gradient is either not defined or is equal to zero. Thus, gradient-based optimization methods cannot be applied. The most common alternative is to  utilize the \emph{logistic regression loss function}, as an approximation of the prediction error and solve the following  unconstrained convex optimization problem
\begin{equation} \label{min.logreg}
\begin{aligned}
\min_{w \in \mathbb{R}^d}  \hat F_{log}(w) = \frac{1}{n}\sum_{i=1}^{n} \log \left(1+ e^{-y_i \cdot f(x_i;w)}\right ) + \lambda r(w) ,
\end{aligned}
\end{equation}
where $ \lambda r(w)$ is the regularization term, with  $r(\cdot)=\|\cdot\|_1$ or $r(\cdot)=\|\cdot\|_2$ as possible examples. 
While the  logistic loss function and its derivatives  have closed form solution, the effort of computing  $\hat F_{log}(w)$ and its gradients has  a linear dependency on $n$.  One can utilize numerous stochastic gradient schemes to reduce the per-iteration complexity of optimizing  $\hat F_{log}(w)$, however, the approach we propose here achieves  better result faster and with a method that requires almost no tuning.

We now discuss the AUC function as an alternative to the prediction accuracy as the quality measure of a classifier.
Let $\mathcal{X}^+$ and $\mathcal{X}^-$ denote the space of the positive and negative input vectors, respectively. 
Let $\mathcal{S}^+ := \{x: (x,y)\in \mathcal{S},~y=+1\}=\{x_1^+, \ldots, x_{n^+}^+\}$ and $\mathcal{S}^- := \{x: (x,y)\in \mathcal{S},~y=-1\}=\{x_1^-, \ldots, x_{n^-}^-\}$ be the sample sets of positive and negative examples, 
 so that $x_i^{+}$ is an i.i.d. observation of the random variable $X^+$ from  $\mathcal{X}^+$ and $x_j^{-}$ is an i.i.d. observation of the random variable $X^-$ from $\mathcal{X}^-$.  
 Then
 %, given the joint probability distribution $P_{X^+,X^-}\left(x^+,x^-\right)$, 
 the expected AUC function of a  classifier $f(x;w)$ is defined as 
\begin{equation} \label{def.AUC}
\begin{aligned}
F_{AUC}(w) &= \mathbb{E}_{{X}^+,{X}^-} \left [1-\mathbbm{1}\left (f\left(X^+;w\right)- f\left(X^-;w\right)\right ) \right ]  \\
& = P\left (f\left(X^+;w\right)- f\left(X^-;w\right) \geq 0 \right ). 
%=~& \int_{\mathcal{X}^+} \int_{\mathcal{X}^-} P_{X^+,X^-}\left(x^+,x^-\right) \\
%& \cdot \mathbbm{1}\left [f\left(x^+;w\right)> f\left(x^-;w\right)\right ] dx^- dx^+.
\end{aligned}
\end{equation}
The reciprocal of AUC is the ranking loss, defined as
\begin{equation} \label{def.RL}
F_{rank}(w) = P\left (f\left(X^+;w\right)- f\left(X^-;w\right) < 0 \right ). 
\end{equation}

The empirical ranking loss is the sample average approximation of \eqref{def.RL} which is defined as
\begin{equation} \label{def.ERL}
\hat F_{rank}(w) =\frac{\sum_{i=1}^{n^+} \sum_{j=1}^{n^-}  \mathbbm{1}  \left (f( x_i^+;x) - f( x_j^-;w) \right )}{n^+  \cdot n^-},
\end{equation}
where $n^+=|\mathcal{S}^+|$, $n^-=|\mathcal{S}^-|$ and  $\mathbbm{1} (\cdot)$ is defined as in \eqref{step_0,1}.

Similarly to the empirical risk minimization, the gradient of this  function is either zero or not defined, thus, gradient-based optimization methods cannot be applied directly. 
 Various techniques have been proposed to approximate the ranking loss with a surrogate function. 
In \cite{mozer}, the indicator function $ \mathbbm{1}[\cdot]$ in \eqref{def.ERL} is substituted with a  \emph{sigmoid surrogate function}, e.g., $1/ \left({1+e^{-\beta \left(f(x^+;w) - f(x^-;w)\right)}} \right)$ and a gradient descent algorithm is applied to this smooth approximation. The choice of the parameter $\beta$  in the  sigmoid function definition significantly affects the output of this approach; although a large value of $\beta$ renders a closer approximation of  the step function, it also results in large oscillations  of the gradients, which in turn can cause numerical issues in the gradient descent algorithm. Similarly, as is discussed in \cite{rudin}, \emph{pairwise exponential loss} and \emph{pairwise logistic loss} can be utilized as convex smooth surrogate functions of the indicator function $ \mathbbm{1}[\cdot]$. However, due to the required pairwise comparison of the value of $f(\cdot;w)$, for each positive and negative pair, the complexity of computing function value as well as the gradient is of the order of $\mathcal{O}\left(n^+n^- \right)$, which can be very expensive. In \cite{steck}, the following \emph{pairwise hinge loss} has been used as a surrogate function, 
\begin{equation} \label{AUC_hinge}
\begin{aligned}
&\hat {F}_{hinge}\left(w\right) \\ 
&= \frac{\sum_{i=1}^{n^+} \sum_{j=1}^{n^-} \max \left \{0, 1- \left( f(x^-_j;w) - f(x^+_i;w)  \right)\right \} }{n^+  \cdot n^-}.
\end{aligned}
\end{equation}
The advantage of pairwise hinge loss over other alternative approximations, such as pairwise logistic loss, lies in the fact that the function values as well as the subgradients of pairwise hinge loss can be computed in roughly $\mathcal{O}\left( n \log(n) \right)$ time, where $n = n^+ + n^-$, by first sorting all values $f(x^-_j;w)$ and $f(x^+_i;w)$. 
Stochastic gradient schemes for  surrogate ranking loss objectives are not as well developed as those ERM \cite{Fan_AUC}. 

In this paper, we propose to optimize alternative smooth approximations of expected error and expected ranking loss, which display good accuracy and also have low computational cost. Towards that end, in the next section, we show that,  under some assumptions on the data,   the
 expected  error and expected ranking loss  of a linear classifier are both smooth functions with closed form expressions. 

%*********
% Section
%*********
\section{New Smooth Approximations of Expected  Error and Ranking Loss for Linear Classifiers}\label{sec.ER.AUC} 
Let us first consider  the expected error expression \eqref{expectedRisk} specifically for the case when $f(x;w)=w^\top x$.
We have the following simple lemma. \\

\begin{lemma} \label{ER_prob}
Given the \textit{prior probabilities} $P(Y=+1)$ and $P(Y=-1)$ we can write 
\begin{equation*}
\begin{aligned}
F_{01}(w) =~& P({Y} \cdot w^\top {X} <0) \\
=~& P\left( w^\top {X}^+ \leq 0\right) P\left({Y} = +1\right) \\
&+ \left(1-P\left(w^\top {X}^- \leq 0\right)\right) P\left({Y} = -1\right),
\end{aligned}
\end{equation*}
where $X^+$ and $X^-$ are random variables from positive and negative classes, respectively. 
\end{lemma}

\begin{proof} 
Note that we can split the whole set $\left \{(X,Y)~:~Y \cdot w^\top X <0 \right \} \subset \mathcal{X} \times \mathcal{Y}$ into two disjoint sets as the following:
\begin{equation*}
\left \{(X,Y)~:~Y \cdot w^\top X <0 \right \} = \left \{(X^+,+1)~:~w^\top X^+ <0 \right \} \cup \left \{(X^-,-1)~:~ w^\top X^- \geq 0 \right \}.
\end{equation*}
Then from the definition of $F_{01}$ we have 
\begin{equation*}
\begin{aligned}
F_{01}(w) =~& P\left({Y} \cdot w^\top {X}<0\right) \\
=~& P\left({Y} \cdot w^\top {X}<0 \cap {Y} = +1 \right) + P\left({Y} \cdot w^\top {X}<0 \cap {Y} = -1 \right) \\
=~& P\left({Y} \cdot w^\top {X}<0 \mid {Y} = +1 \right) P\left({Y} = +1\right) + P\left({Y} \cdot w^\top {X}<0 \mid {Y} = -1 \right) P\left({Y} = -1\right) \\
=~& P\left( w^\top {X}^+<0\right) P\left({Y} = +1\right) + P\left(w^\top {X}^->0\right) P\left({Y} = -1\right) \\
=~& P\left( w^\top {X}^+\leq 0\right) P\left({Y} = +1\right) + \left(1-P\left(w^\top {X}^- \leq 0\right)\right) P\left({Y} = -1\right).
\end{aligned}
\end{equation*} 
\end{proof} 
%\begin{proof}
%Note that we can split the whole set $\left \{(X,Y)~:~Y \cdot w^\topX <0 \right \} \subset \mathcal{X} \times \mathcal{Y}$ into two disjoint sets as the following:
%\begin{equation*}
%\begin{aligned}
%\left \{(X,Y)~:~Y \cdot w^\topX <0 \right \} =& \left \{(X^+,+1)~:~w^\topX^+ <0 \right \} \\
% \cup& \left \{(X^-,-1)~:~ w^\topX^- \geq 0 \right \}.
%\end{aligned}
%\end{equation*}
%Now, by using  \eqref{expectedRisk_prob} we will have:
%\begin{equation*}
%\begin{aligned}
%F_{error}(w) =~& P\left({Y} \cdot w^\top {X}<0\right) \\
%=~& P\left({Y} \cdot w^\top {X}<0 \cap {Y} = +1 \right) \\ &+ P\left({Y} \cdot w^\top {X}<0 \cap {Y} = -1 \right) \\
%=~& P\left({Y} \cdot w^\top {X}<0 \mid {Y} = +1 \right) P\left({Y} = +1\right) \\
%&+ P\left({Y} \cdot w^\top {X}<0 \mid {Y} = -1 \right) P\left({Y} = -1\right) \\
%=~& P\left( w^\top {X}^+<0\right) P\left({Y} = +1\right) \\ &+ P\left(w^\top {X}^->0\right) P\left({Y} = -1\right) \\
%=~& P\left( w^\top {X}^+\leq 0\right) P\left({Y} = +1\right) \\ &+ \left(1-P\left(w^\top {X}^- \leq 0\right)\right) P\left({Y} = -1\right).
%\end{aligned}
%\end{equation*}
%\end{proof}

Based on this result   $F_{01}(w)$ is a continuous and smooth function everywhere, except $w=0$, if   the Cumulative Distribution Function (CDF)  of the  random variable $w^\top {X}$ is a continuous smooth function.  In general, it is possible to derive smoothness of the CDF of $w^\top {X}$ for a variety of distributions, which will imply that, in principal, continuous optimization techniques can be applied to optimize $F_{01}(w)$. However, to use gradient-based methods it is necessary to obtain an estimate of the gradient 
of $F_{01}(w)$. In this paper, we will show that under the assumption that $w^\top {X}$ obeys normal distributions for both  positive and negative classes,  $F_{01}(w)$ and its derivatives  have  closed form expressions. First, however, we will motivate this assumption and derive the moments of the distribution of  $w^\top {X}$. 

Since the  multivariate normal distribution is closed under linear transformation  \cite{tong} we have the following result. \\

\begin{lemma} Let ${X} \sim \mathcal{N}\left(\mu, \Sigma\right)$ be a random Gaussian vector and $w$ be a fixed vector.  Then the random variable $Z=w^\top X$  obeys normal distribution as follows. 
\begin{equation}
{Z} \sim \mathcal{N}\left(w^T \mu, w^T \Sigma w \right)
\end{equation}
\end{lemma}

Let us now consider the case when $X\in R^d$ is not a Gaussian vector, but some random vector with mean $\mu$ and covariance $\Sigma$. Then $Z=w^\top X=\sum_{i=1}^d w_iX_i$ is a sum of $d$ random variables $w_iX_i$ $i=1, \ldots d$.
Clearly $Z$ is a random variable with mean $w^T \mu$ and  variance $w^T \Sigma w $. While in general the distribution of $Z$ is unknown, we argue that for many data sets and for $w$ that we encounter in the training process this distribution is {\em close to normal}, due to Central Limit Theorem \cite{patrick} and its many variants for dependent variables (see e.g. \cite{fisher}). Indeed,  for successful learning  it is better to have features that are nearly independent, and while features of a data vector $X$ are not typically individually independent, they often can be grouped into blocks of variables, so that the blocks are (nearly) independent, or these features otherwise obey some structured dependence that can be exploited by the CLT. 
This observation motivates our derivation of the closed form of $F_{01}(w)$.

 Let $X^+$ and $X^-$ be the random vectors from the positive and negative class, respectively,
  with their means and covariances denoted by  $\mu^+$, $\mu^-$ and  $\Sigma^{+}$ and $\Sigma ^{-}$, respectively.  
 Let $P({Y}=+1)$ and $P({Y}=-1)$ denote the probabilities that a random data vector $X$ belongs to the positive or negative class respectively. \\

\begin{theorem} \label{theorem.main.0_1}
Suppose that for a given $w\neq 0$, $w^\top X^+$  and  $w^\top X^-$ obey normal distributions as follows
\begin{equation} \label{original.normal}
\begin{aligned}
&Z^+=w^\top {X}^+ \sim \mathcal{N}\left(\mu_{Z^+}, \sigma^2_{Z^+}\right)~~~ \text{and}\\
&Z^-=w^\top {X}^- \sim \mathcal{N}\left(\mu_{Z^-}, \sigma^2_{Z^-}\right),
\end{aligned}
\end{equation}
where $\mu_{Z^+} = w^T \mu^+$, $\sigma_{Z^+} = \sqrt{w^T \Sigma^+ w}$, $\mu_{Z^-} = w^T \mu^-$, and $\sigma_{Z^-} = \sqrt{w^T \Sigma^- w}$. 

Then, $F_{01}$ defined in \eqref{expectedRisk}  for $f(x;w)=w^\top x$ equals  $F_{n01}(w)$ which is defined as follows
\begin{equation} \label{smooth.zero.one}
\begin{aligned}
F_{n01}(w) =~& P({Y}=+1)  \left (1-\phi \left (\mu_{Z^+}/\sigma_{Z^+}\right ) \right ) \\
&+ P({Y}=-1)  \phi \left (\mu_{Z^-}/\sigma_{Z^-}\right ),
\end{aligned}
\end{equation}
%$\mu_{Z^+} = w^\top \mu^+$, $\sigma_{Z^+} = \sqrt{w^\top \Sigma^+ w}$, $\mu_{Z^-} = w^\top \mu^-$, and $\sigma_{Z^-} = \sqrt{w^\top \Sigma^- w}$, and 
where $\phi$ is the CDF of the standard normal distribution, that is $\phi(v) = \int_{-\infty}^v\frac{1}{\sqrt{2 \pi}}\exp({-\frac{1}{2} t^2}) dt$, for $\forall v \in \mathbb{R}$.
\end{theorem} 

\begin{proof}  Let us define  random variables $Z^+$ and $Z^-$ as is stated in the theorem. Then we have 
\begin{equation*}
\begin{aligned}
F_{n01}(w) &= P\left ({Y} \cdot w^\top {X}<0\right ) \\
&= \left (1- \phi\left({ \mu_{Z^+}}/{\sigma_{Z^+}}\right)\right ) P\left ({Y} = +1\right ) + \phi\left({\mu_{Z^-}}/{\sigma_{Z^-}}\right) P\left ({Y} = -1\right ).
\end{aligned}
\end{equation*}
where $\mu_{Z^+} = w^\top \mu^+$, $\sigma_{Z^+} = \sqrt{w^\top \Sigma^+ w}$, $\mu_{Z^-} = w^\top \mu^-$, and $\sigma_{Z^-} = \sqrt{w^\top \Sigma^- w}$. 
\end{proof}

%\begin{proof} Let us define  random variables $Z^+$ and $Z^-$ as the follows
%\begin{equation*}
%Z^+ = w^\top {X}^+,~~\text{and}~~Z^- = w^\top {X}^-.
%\end{equation*}
%From \eqref{original.normal} and using Theorem \ref{th.normal.close} we have
%\begin{equation*}
%\begin{aligned}
%{Z}^+ &\sim \mathcal{N}\left(w^\top \mu^+, w^\top \Sigma^+ w \right)~~\text{and} \\
%{Z}^- &\sim \mathcal{N}\left(w^\top \mu^-, w^\top \Sigma^- w \right).
%\end{aligned}
%\end{equation*}
%Then, by using Lemma \ref{ER_prob} we conclude the following 
%\begin{equation*}
%\begin{aligned}
%F_{error}(w) &= P\left ({Y} \cdot w^\top {X}<0\right ) \\
%&= \left (1- \phi\left({ \mu_{Z^+}}/{\sigma_{Z^+}}\right)\right ) P\left ({Y} = +1\right ) \\
%& + \phi\left({\mu_{Z^-}}/{\sigma_{Z^-}}\right) P\left ({Y} = -1\right ).
%\end{aligned}
%\end{equation*}
%where $\mu_{Z^+} = w^\top \mu^+$, $\sigma_{Z^+} = \sqrt{w^\top \Sigma^+ w}$, $\mu_{Z^-} = w^\top \mu^-$, and $\sigma_{Z^-} = \sqrt{w^\top \Sigma^- w}$. 
%\end{proof}

In Theorem \ref{col.error.short} we  derive the closed form solution of the gradient of $F_{01}(w)$. First we provide expression for the derivative of the cumulative function $\phi\left (g(w)\right )$, where $g(w) = {w^\top {\mu}}/{\sqrt{w^\top {\Sigma} w}}$. \\

\begin{lemma} \label{lemma.derivative.phi}
The first derivative of the cumulative function $\phi\left (g(w)\right )$  with $g(w) = \frac{w^\top {\mu}}{\sqrt{w^\top {\Sigma} w}}$ is
\begin{equation*}
\begin{aligned}
\frac{d}{dw} \phi(g(w)) & = \frac{1}{\sqrt{2 \pi}} \exp\left ({-\frac{1}{2} \left ({\frac{w^\top {\mu}}{\sqrt{w^\top {\Sigma} w}}}\right )^2}\right ) \\
&  \left ( \frac{\sqrt{w^\top {\Sigma} w} \cdot {\mu} - {\frac{w^\top {\mu}}{\sqrt{w^\top {\Sigma} w}}} \cdot {\Sigma} w}{w^\top {\Sigma} w} \right ).
\end{aligned}
\end{equation*}
\end{lemma} 

\begin{proof} 
Note that based on the chain rule we have 
\begin{equation}\label{chain.rule}
\frac{d}{dw} \phi \left (g(w)\right ) = \phi'(g(w)) g'(w).
\end{equation}
By substituting 
\begin{equation*}
\phi'(v) = \frac{d}{dv} \int_{-\infty}^{v} \frac{1}{\sqrt{2 \pi}}\exp\left ({-\frac{1}{2} t^2}\right ) dt =  \frac{1}{\sqrt{2 \pi}}\exp\left ({-\frac{1}{2} v^2}\right )~~~\text{and }
\end{equation*}
\begin{equation*}
g'(w) =  \frac{\sqrt{w^\top \hat{\Sigma} w} \cdot \hat{\mu} - {\frac{w^\top \hat{\mu}}{\sqrt{w^\top \hat{\Sigma} w}}} \cdot \hat{\Sigma} w}{w^\top \hat{\Sigma} w},
\end{equation*}
in \eqref{chain.rule} we conclude the result.
\end{proof} 
%
%\begin{proof}
%Note that based on the chain rule we have 
%\begin{equation}\label{chain.rule}
%\frac{d}{dw} \phi \left (f(w)\right ) = \phi'(g(w)) g'(w).
%\end{equation}
%By substituting 
%\begin{equation*}
%\begin{aligned}
%\phi'(x) &= \frac{d}{dx} \int_{-\infty}^{x} \frac{1}{\sqrt{2 \pi}}\exp\left ({-\frac{1}{2} t^2}\right ) dt \\
%&=  \frac{1}{\sqrt{2 \pi}}\exp\left ({-\frac{1}{2} x^2}\right )~~~\text{and }
%\end{aligned}
%\end{equation*}
%\begin{equation*}
%g'(w) =  \frac{\sqrt{w^\top \hat{\Sigma} w} \cdot \hat{\mu} - {\frac{w^\top \hat{\mu}}{\sqrt{w^\top \hat{\Sigma} w}}} \cdot \hat{\Sigma} w}{w^\top \hat{\Sigma} w} 
%\end{equation*}
%in \eqref{chain.rule} we conclude the result.
%\end{proof}
%

%\begin{theorem} \label{col.error}
%%Using the result of Lemma \ref{lemma.derivative.phi}, 
%The derivative of the smooth function $F_{error}(w)$ is defined as
%\begin{equation*}
%\begin{aligned}
%&\frac{d}{dw} F_{error}(w) \\
%=~& P\left (Y = -1\right ) \frac{1}{\sqrt{2 \pi}} \exp\left ({-\frac{1}{2} \left ({\frac{w^\top {\mu}^-}{\sqrt{w^\top {\Sigma}^- w}}}\right )^2}\right ) \\
%& \left ( \frac{\sqrt{w^\top {\Sigma}^- w} \cdot {\mu}^- -  {\frac{w^\top {\mu}^-}{\sqrt{w^\top {\Sigma}^- w}}} \cdot {\Sigma}^- w}{w^\top {\Sigma}^- w} \right ) \\
%&-P\left (Y = +1\right ) \frac{1}{\sqrt{2 \pi}} \exp\left ({-\frac{1}{2} \left ({\frac{w^\top {\mu}^+}{\sqrt{w^\top {\Sigma}^+ w}}}\right )^2}\right ) \\
%& \left ( \frac{\sqrt{w^\top {\Sigma}^+ w} \cdot {\mu}^+ -  {\frac{w^\top {\mu}^+}{\sqrt{w^\top {\Sigma}^+ w}}} \cdot {\Sigma}^+ w}{w^\top {\Sigma}^+ w} \right ).
%\end{aligned}
%\end{equation*}
%\end{theorem}

\begin{theorem} \label{col.error.short}
%Using the result of Lemma \ref{lemma.derivative.phi}, 
Under conditions of Theorem \ref{theorem.main.0_1} the gradient of the  $F_{01}(w)$ for $w\neq 0$ is equal to the gradient of $F_{n01}(w)$ which is
\begin{equation}\label{eq:col.error.short}
\begin{aligned}
\nabla_w F_{n01}(w) &= P\left (Y = -1\right ) \frac{1}{\sqrt{2 \pi}} \exp\left (-\frac{1}{2} \left (\frac{\mu_{Z^-}}{\sigma_{Z^-}}\right )^2\right ) \\
 & \cdot \left ( \frac{\sigma_{Z^-}{\mu}^- -  \frac{\mu_{Z^-}}{\sigma_{Z^-}}  \cdot \Sigma^- w}{\sigma_{Z^-}^2} \right ) \\
&- P\left (Y = +1\right ) \frac{1}{\sqrt{2 \pi}} \exp\left (-\frac{1}{2} \left (\frac{\mu_{Z^+}}{\sigma_{Z^+}}\right )^2\right ) \\
&\cdot  \left ( \frac{\sigma_{Z^+} {\mu}^+ -  \frac{\mu_{Z^+}}{\sigma_{Z^+}} \cdot  \Sigma^+ w}{\sigma_{Z^+}^2} \right ),
% \\
%&-P\left (Y = +1\right ) \frac{1}{\sqrt{2 \pi}} \exp\left ({-\frac{1}{2} \left ({\frac{w^\top {\mu}^+}{\sqrt{w^\top {\Sigma}^+ w}}}\right )^2}\right ) \\
%& \left ( \frac{\sqrt{w^\top {\Sigma}^+ w} \cdot {\mu}^+ -  {\frac{w^\top {\mu}^+}{\sqrt{w^\top {\Sigma}^+ w}}} \cdot {\Sigma}^+ w}{w^\top {\Sigma}^+ w} \right ).
\end{aligned}
\end{equation}
where $\mu_{Z^-}$, $\sigma_{Z^-}$, $\mu_{Z^+}$, and $\sigma_{Z^+}$ are defined in Theorem \ref{theorem.main.0_1}.
\end{theorem}

\begin{proof} 
Theorem \ref{col.error.short} is an immediate corollary of  the result of Lemma \ref{lemma.derivative.phi}.
\end{proof}

We now apply analogous  derivation to \eqref{def.RL} in order to obtain a closed form expression of $F_{rank}(w)$ and its gradient under the assumption that  $w^\top \left (X^- - X^+\right )$ obeys normal distribution. 
As in the case of $F_{01}(w)$, the smoothness of $F_{rank}(w)$ for $w\neq 0$ follows from the smoothness of the CDF of  $w^\top \left (X^- - X^+\right ) $. 
Let us assume that ${X}^+$ and ${X}^-$ have a joint  distribution, with the following mean  and  covariance 
\begin{equation*}
\mu = \begin{pmatrix} \mu^+ \\ \mu^- \end{pmatrix}~~\text{and}~~~ \Sigma = 
\begin{pmatrix} \Sigma^{++} & \Sigma^{+-} \\ \Sigma^{-+} & \Sigma^{--} \end{pmatrix}.
\end{equation*}

From  \cite{tong} we have\\

\begin{theorem} \label{T2}
If the distribution of ${X}^+$ and ${X}^-$ is Gaussian then, for any $w\in R^d$,
\begin{equation} \label{z}
Z = w^\top \left({X}^--{X}^+\right) \sim \mathcal{N}\left(\mu_Z, \sigma^2_Z\right),~~~\text{where}
\end{equation}
\begin{equation} \label{zp}
\begin{aligned}
 \mu_Z &= w^\top\left(\mu^- - \mu^+\right)~~\text{and}~~\\
\sigma_Z &= \sqrt{w^\top \left( \Sigma^{--} + \Sigma^{++} - \Sigma^{-+} - \Sigma^{+-} \right)w}.
\end{aligned}
\end{equation}
\end{theorem}

In the  case when ${X}^+$ and ${X}^-$ are not Gaussian we again rely on the Central Limit Theorem to argue that
 $Z=w^\top (X^+-X^-)=\sum_{i=1}^d w_i(X_i^+-X_i^-)$ is approximately normal, with  mean 
 $w^\top\left(\mu^- - \mu^+\right)$ and variance  $w^\top \left( \Sigma^{--} + \Sigma^{++} - \Sigma^{-+} - \Sigma^{+-} \right)w$. 

We now derive the formulas for  $F_{rank}(w)$ and its gradient. \\

\begin{theorem} \label{T3}
 Assume that for  a given vector $w\neq 0$ the random variable $Z = w^\top \left({X}^--{X}^+\right)$ obeys normal distribution
$\mathcal{N}\left(\mu_Z, \sigma^2_Z\right)$ with $\mu_Z$ and  $\sigma_Z$ defined as in \eqref{zp}. 
 Then $F_{rank}(w) = F_{nrank}(w) $ with
\begin{equation} \label{F.AUC.smooth}
F_{nrank}(w) = 1 - \phi\left(\frac{\mu_Z}{\sigma_Z}\right),
\end{equation}
where $\phi$ is the CDF of the standard normal distribution, as in Theorem  \ref{theorem.main.0_1}. 
\end{theorem} 
\begin{proof} 
First, consider that
\begin{equation}
\begin{aligned}
F_{AUC}(w) &= 1-\mathbb{E}_{\mathcal{X}^+,\mathcal{X}^-} \left [\mathbbm{1}\left [f\left (X^+;w\right )> f\left (X^-;w\right )\right ] \right ]  \\
 &= 1 - P\left (w^T {X}^+ >w^T {X}^-\right ) \\
& = 1 - P\left (w^T \left (X^- - X^+\right ) < 0 \right ).
\end{aligned}
\end{equation}

Now, by using Theorem \ref{T2}  we conclude that
\begin{equation*}
\begin{aligned}
 F_{rank}(w) &= 1 - P(w^\top \left (X^- - X^+) < 0 \right ) = 1 - P\left(Z \leq 0\right) \\
&= 1 - P\left(\frac{Z - \mu_Z}{\sigma_Z} \leq \frac{ - \mu_Z}{\sigma_Z}\right) = 1 - \phi\left(\frac{\mu_Z}{\sigma_Z}\right),
\end{aligned}
\end{equation*}
where the random variable $Z$ is defined in \eqref{z}, with the stated mean and standard deviation in \eqref{zp}.
\end{proof}
%\begin{proof}
%From \eqref{F_AUC} and Theorem \ref{T2} we have
%\begin{equation*}
%\begin{aligned}
% F_{AUC}(w) &= 1 - P(w^\top \left (X^- - X^+) < 0 \right ) \\
%&= 1 - P\left(Z \leq 0\right) \\
%&= 1 - P\left(\frac{Z - \mu_Z}{\sigma_Z} \leq \frac{ - \mu_Z}{\sigma_Z}\right)\\
%&= 1 - \phi\left(\frac{\mu_Z}{\sigma_Z}\right),
%\end{aligned}
%\end{equation*}
%
%%\begin{equation*}
%%\begin{aligned}
%% F_{AUC}(w) &= P(w^\top \left (X^+ - X^-) > 0 \right ) \\
%%& = P(Z > 0) = 1 - P\left(Z \leq 0\right) \\
%%&= 1 - P\left(\frac{Z - \mu_Z}{\sigma_Z} \leq \frac{ - \mu_Z}{\sigma_Z}\right)\\
%%& = 1 - \phi\left(\frac{ - \mu_Z}{\sigma_Z}\right) \\
%%&= \phi\left(\frac{\mu_Z}{\sigma_Z}\right),
%%\end{aligned}
%%\end{equation*}
%where the random variable $Z$ is defined in \eqref{z}, with the stated mean and standard deviation in \eqref{zp}.
%\end{proof}

\begin{theorem} \label{col.AUC}
%Using the result of Lemma \ref{lemma.derivative.phi}, and the symmetric property of $\phi(\cdot)$, 
Under conditions of  Theorem \ref{T3} the gradient of  $F_{rank}(w)$ is 
%\begin{equation*}
%\begin{aligned}
%\frac{d}{dw} F_{AUC}(w) =~& - \frac{1}{\sqrt{2 \pi}} \exp\left ({-\frac{1}{2} \left ({\frac{w^\top {\hat{\mu}}}{\sqrt{w^\top {\hat{\Sigma}} w}}}\right )^2}\right ) \\
%&\left ( \frac{\sqrt{w^\top {\hat{\Sigma}} w} \cdot {\hat{\mu}} -  {\frac{w^\top \hat{{\mu}}}{\sqrt{w^\top {\hat{\Sigma}} w}}} \cdot {\hat{\Sigma}} w}{w^\top {\hat{\Sigma}} w} \right ).
%\end{aligned}
%\end{equation*}
%where $\hat{\mu} = \mu^- - \mu^+$ and $\hat{\Sigma} = \Sigma^{--} + \Sigma^{++} - \Sigma^{-+} - \Sigma^{+-}$.
\begin{equation}\label{eq:gradrank}
\begin{aligned}
&\nabla_w F_{rank}(w)= \\&
 - \frac{1}{\sqrt{2 \pi}} \exp\left (-\frac{1}{2} 
\left (\frac{\mu_Z}{\sigma_Z}\right )^2\right )\left 
( \frac{\sigma_Z \cdot \hat{\mu} -  {\frac{\mu_Z}{\sigma_Z}}
\cdot {\hat{\Sigma}} w}{\sigma_Z^2} \right )
\end{aligned}
\end{equation}
where $\hat{\mu} = \mu^- - \mu^+$ and $\hat{\Sigma} = \Sigma^{--} + \Sigma^{++} - \Sigma^{-+} - \Sigma^{+-}$, and $\mu_Z$ and $\sigma_Z$ are defined as in \eqref{zp}.
\end{theorem}
\begin{proof} 
Theorem \ref{col.AUC} is an immediate corollary of  the result of Lemma \ref{lemma.derivative.phi}.
\end{proof}

In the next section, we will apply   \emph{L-BFGS method with Wolfe line-search} \cite{nocedal_BFGS} to optimize functions
$F_{n01}$ and $F_{nrank}$ for a variety of artificial and real data sets and compare the results of this optimization to optimizing  $\hat F_{log}(w)$ and $\hat F_{hinge}(w)$, respectively.  Some of the  standard data sets that we use
do not obey   Gaussian distribution, however, as we will show, our method achieves very good results on most of the data sets.  We believe that this is due to the fact that by using the assumption that $w^\top X$  and $w^\top \left (X^- - X^+\right )$ are nearly Gaussian (rather than the data itself) we obtain useful approximation for the expected error and expected ranking loss. 

Note that $F_{01}$ and $F_{rank}$ as well as their approximations $F_{n01}$ and $F_{nrank}$,  are not well defined for $w=0$. On the other hand, all these functions are invariant to the scale of $w$, that is $F_{n01}(w)=F_{n01}(\alpha w)$ and $F_{nrank}(w)=F_{nrank}(\alpha w)$, for any $\alpha>0$. 
Thus, ideally,  functions $F_{n01}$ and $F_{nrank}$  should be optimized subject to a constraint $\|w\|=1$, however, since this constraint make optimization harder and yet does not have to hold exactly,  instead of imposing it directly, we include a penalty  $\lambda (1-\|w\|^2)^2$.  We can tune this $\lambda$ the same way as the regularization parameter for logistic regression, although in our experiments a value of $0.001$ worked well for most of the data sets.

%*********
% Section
%*********
\section{Numerical Experiments}\label{sec.numerical}

Our first experiment is provided to illustrate  the assumption that $w^\top X^+$, $w^\top X^-$ and $w^\top (X^--X^+)$ may have near normal distributions even when the data distribution itself  not close to Gaussian. In Figure \ref{fig:normal}  we plot
empirical distributions of these  three  random variables  for the {\em a9a} data set (see description of data set below)
whose features are binary encodings of categorical values. We plot these distributions for two choices of  $w$--one used early in the training (left column) and one used close to the end of the training (right column). 
We observe that the random variables $w^\top X^-$, $w^\top X^+$, and $w^\top (X^--X^+)$  have almost perfectly normal distributions in both cases.

% Figure 1
\begin{figure}[tph] 
\begin{subfigure}{0.47\textwidth}
  \centering
  \includegraphics[height=0.8\linewidth,width=\linewidth]{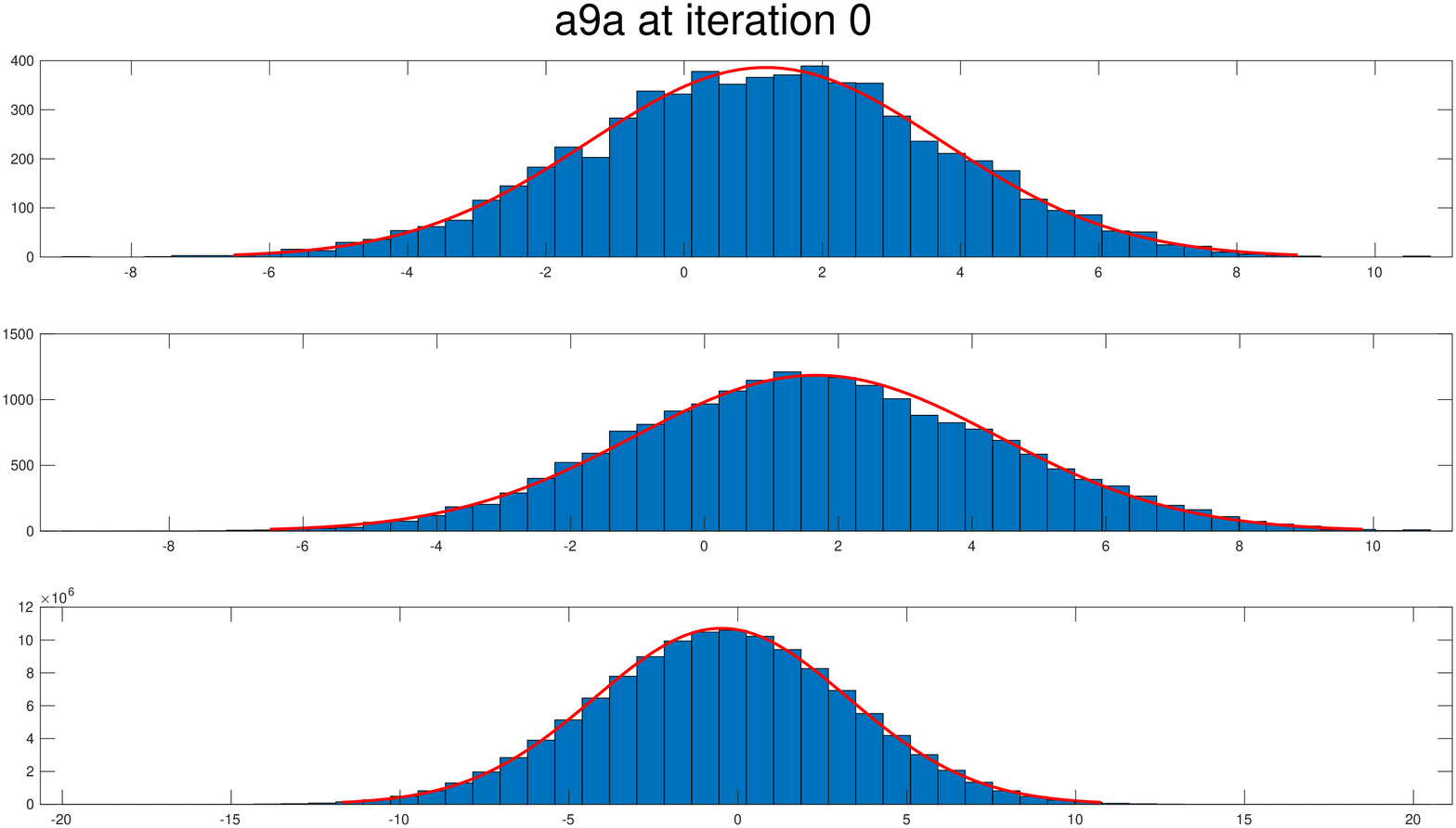}
  \end{subfigure}%
  \hfill
\begin{subfigure}{0.47\textwidth}
  \centering
  \includegraphics[height=0.8\linewidth,width=\linewidth]{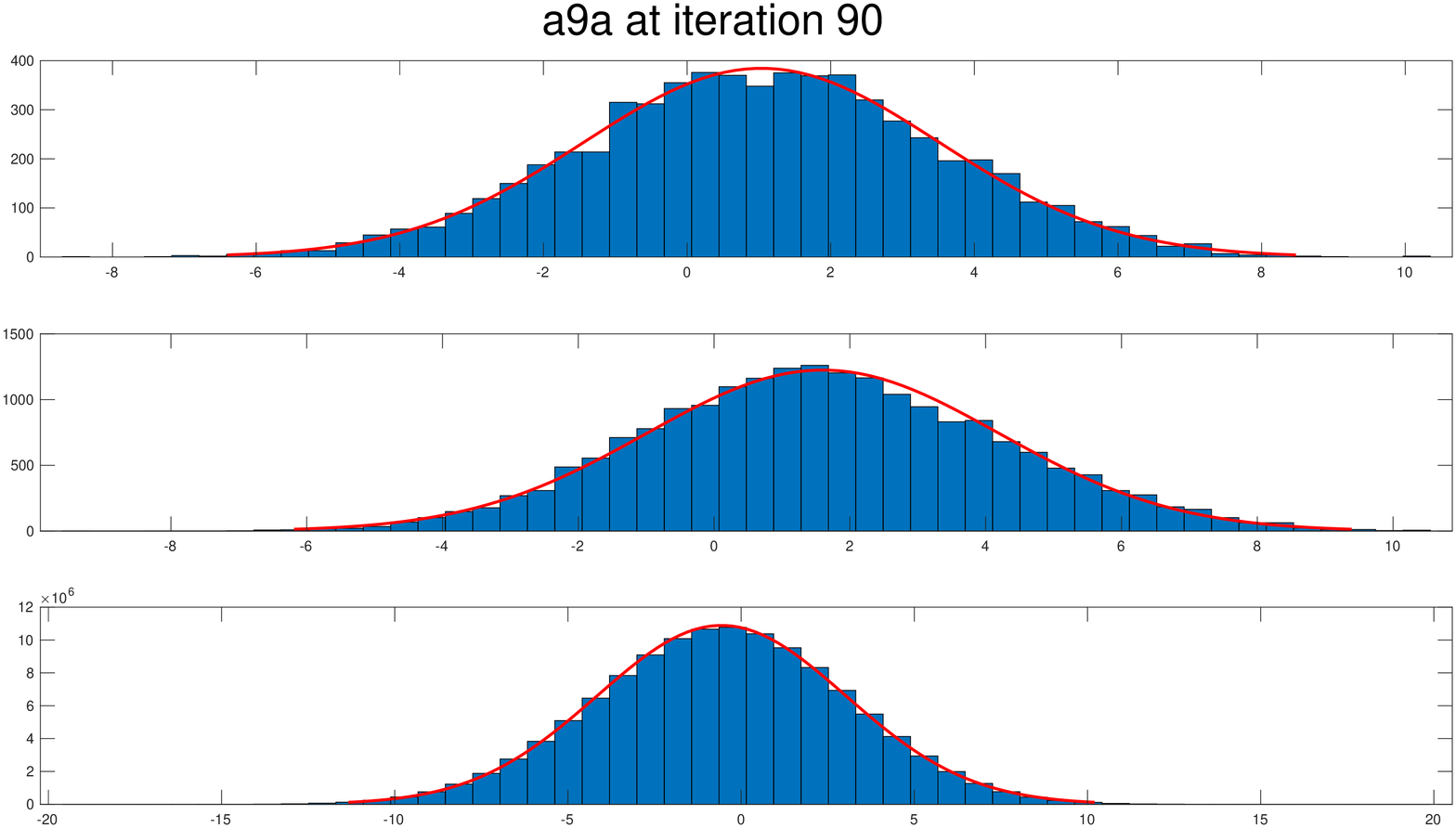}
\end{subfigure} %
\caption{Nearly normal distribution for $w^\top X^+$, $w^\top X^-$ and $w^\top (X^+-X^-)$ at different iterations during training.}
\label{fig:normal}
\end{figure}

We also present the plots for additional four datasets {\em diabetes}, {\em poker}, {\em svm3} and {\em segment} in Figure \ref{fig:normal1}. We observe that not in all cases random variables $w^\top X^-$, $w^\top X^+$, and $w^\top (X^--X^+)$  seem to have near normal distributions although it is more common  for $w^\top (X^--X^+)$. Nevertheless our proposed functions $F_{n01}$ and $F_{nrank}$ still seem to  provide useful approximations to empirical risk, as will be evident from the plots of these functions we present later in this section. 

 %Figure 6
\begin{figure}[tph] 
\begin{subfigure}{0.47\textwidth}
 \centering
  \includegraphics[width=\linewidth]{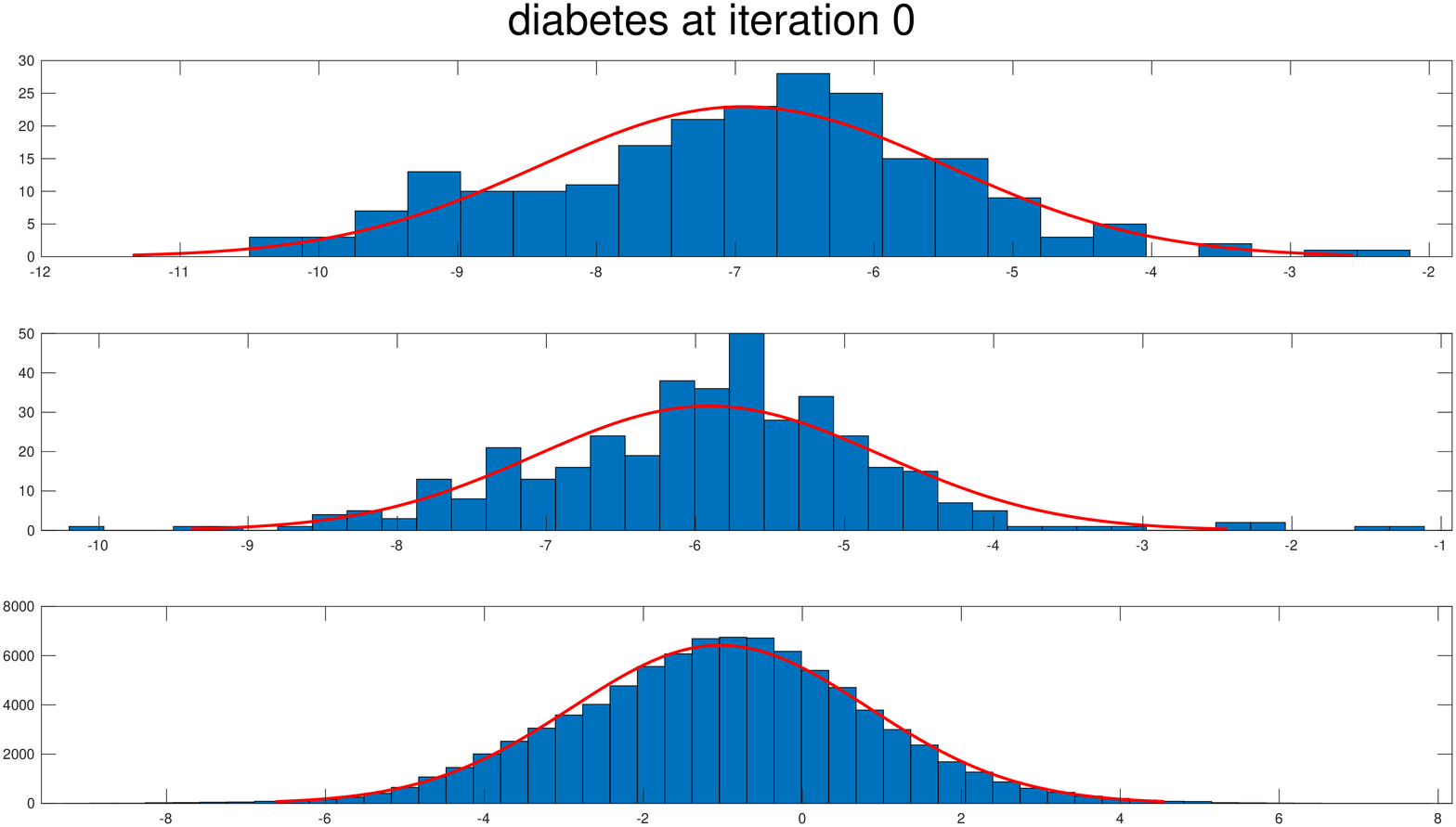}
  \end{subfigure}%
  \hfill
\begin{subfigure}{0.47\textwidth}
  \centering
  \includegraphics[width=\linewidth]{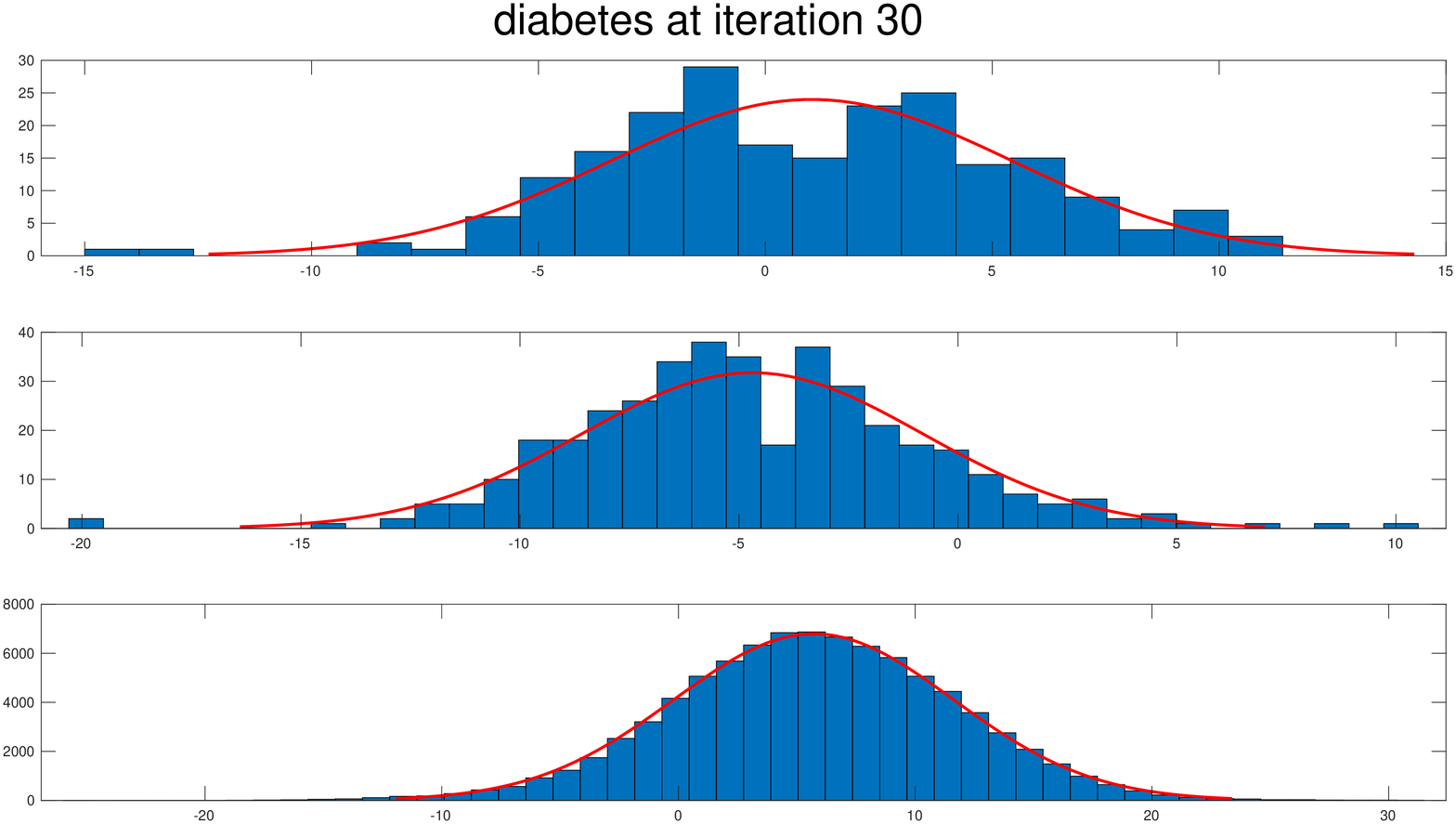}
\end{subfigure} %
\begin{subfigure}{0.47\textwidth}
  \centering
  \includegraphics[width=\linewidth]{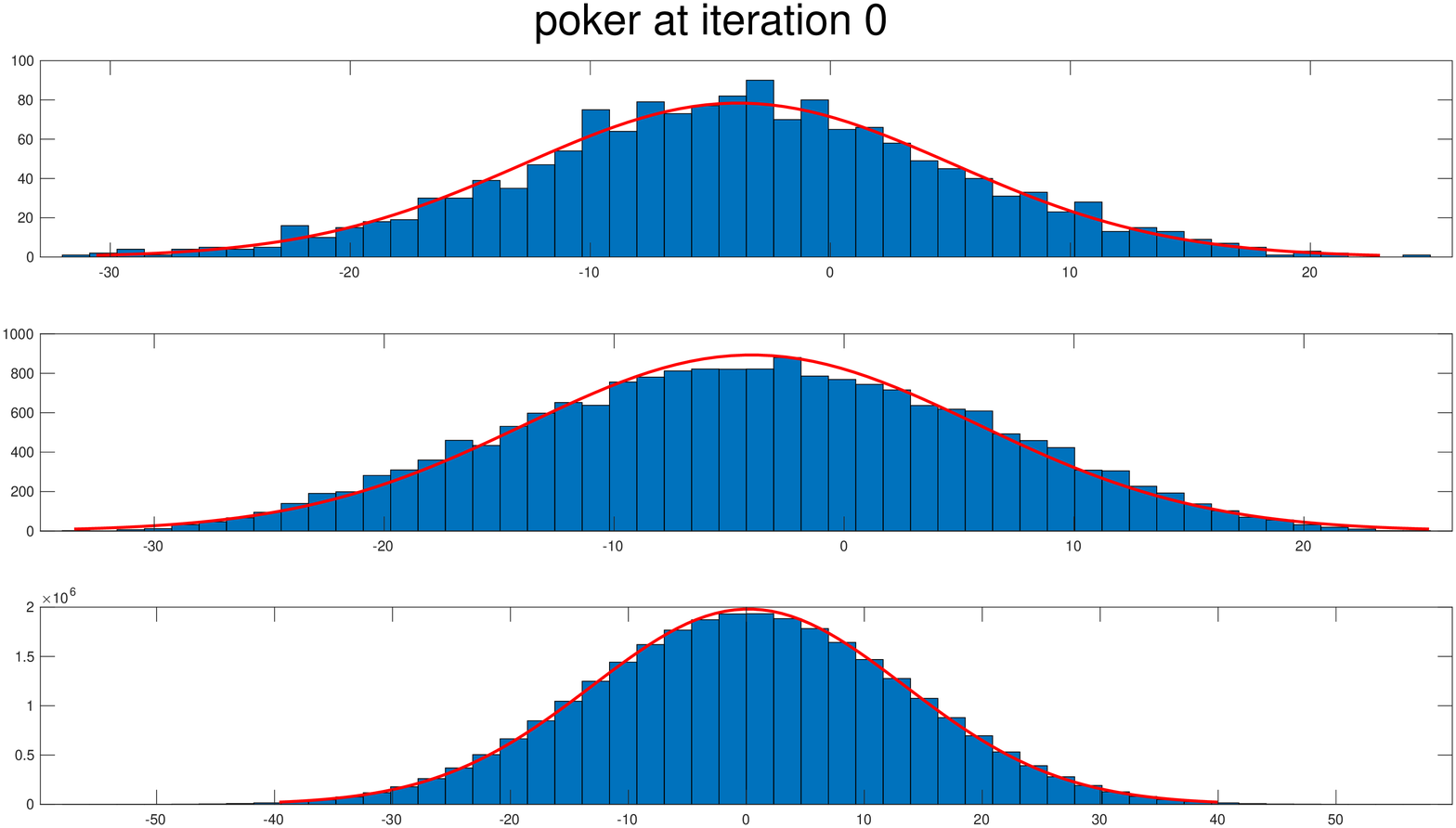}
  \end{subfigure}%
  \hfill
\begin{subfigure}{0.47\textwidth}
  \centering
  \includegraphics[width=\linewidth]{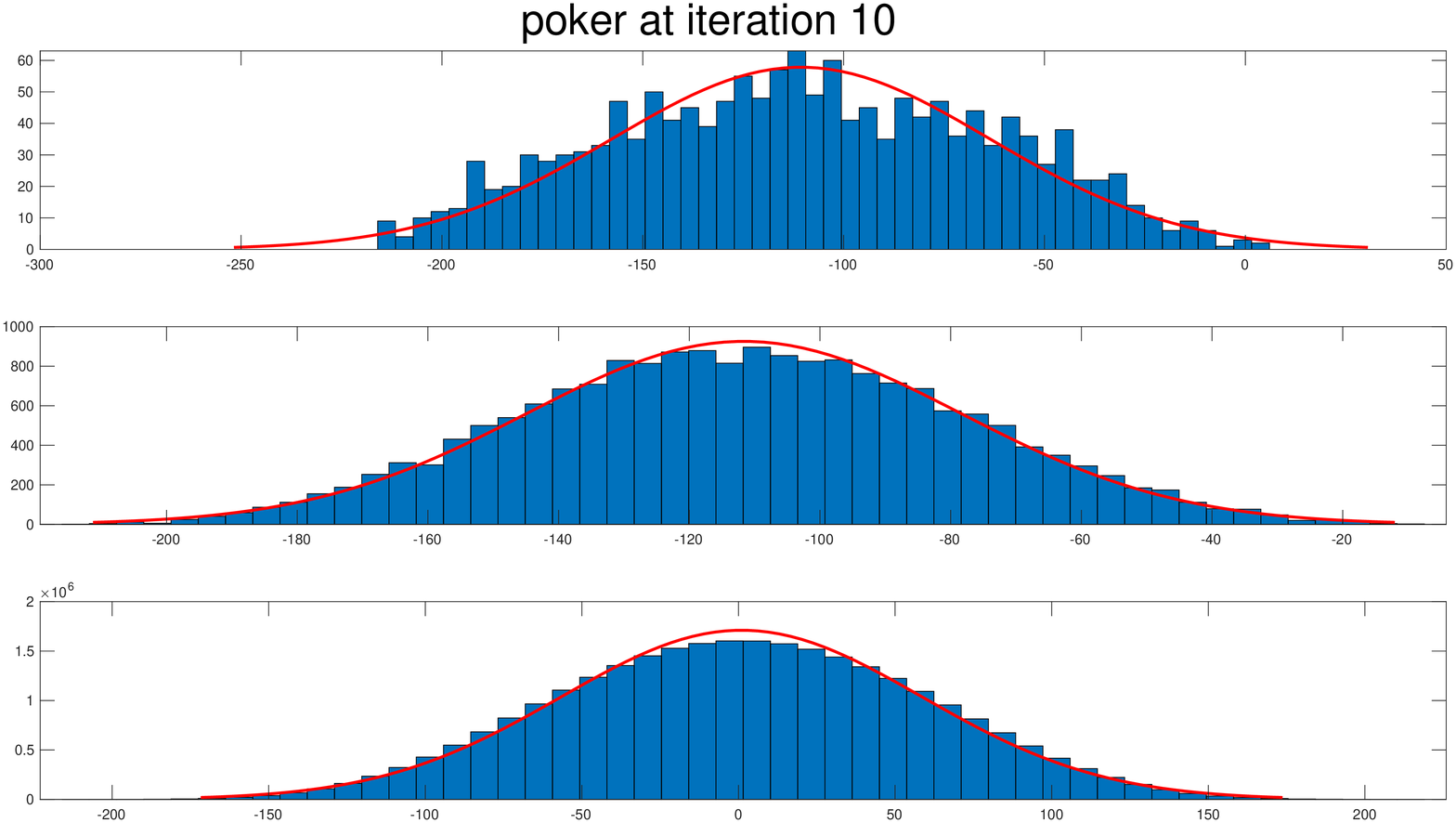}
\end{subfigure} %
\begin{subfigure}[t]{0.47\textwidth}
  \centering
  \includegraphics[width=\linewidth]{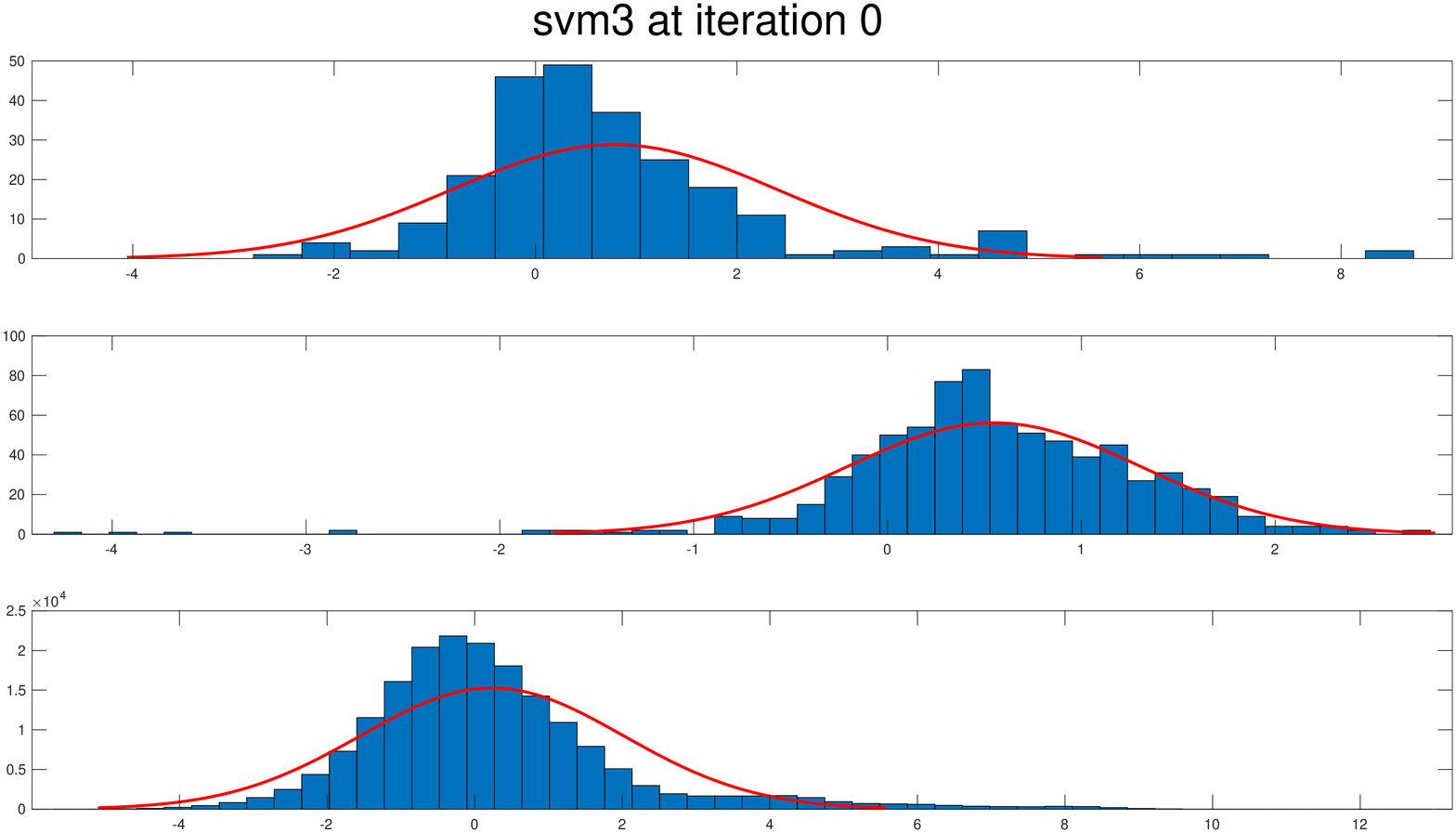}
  \end{subfigure}%
  \hfill
\begin{subfigure}[t]{0.47\textwidth}
  \centering
  \includegraphics[width=\linewidth]{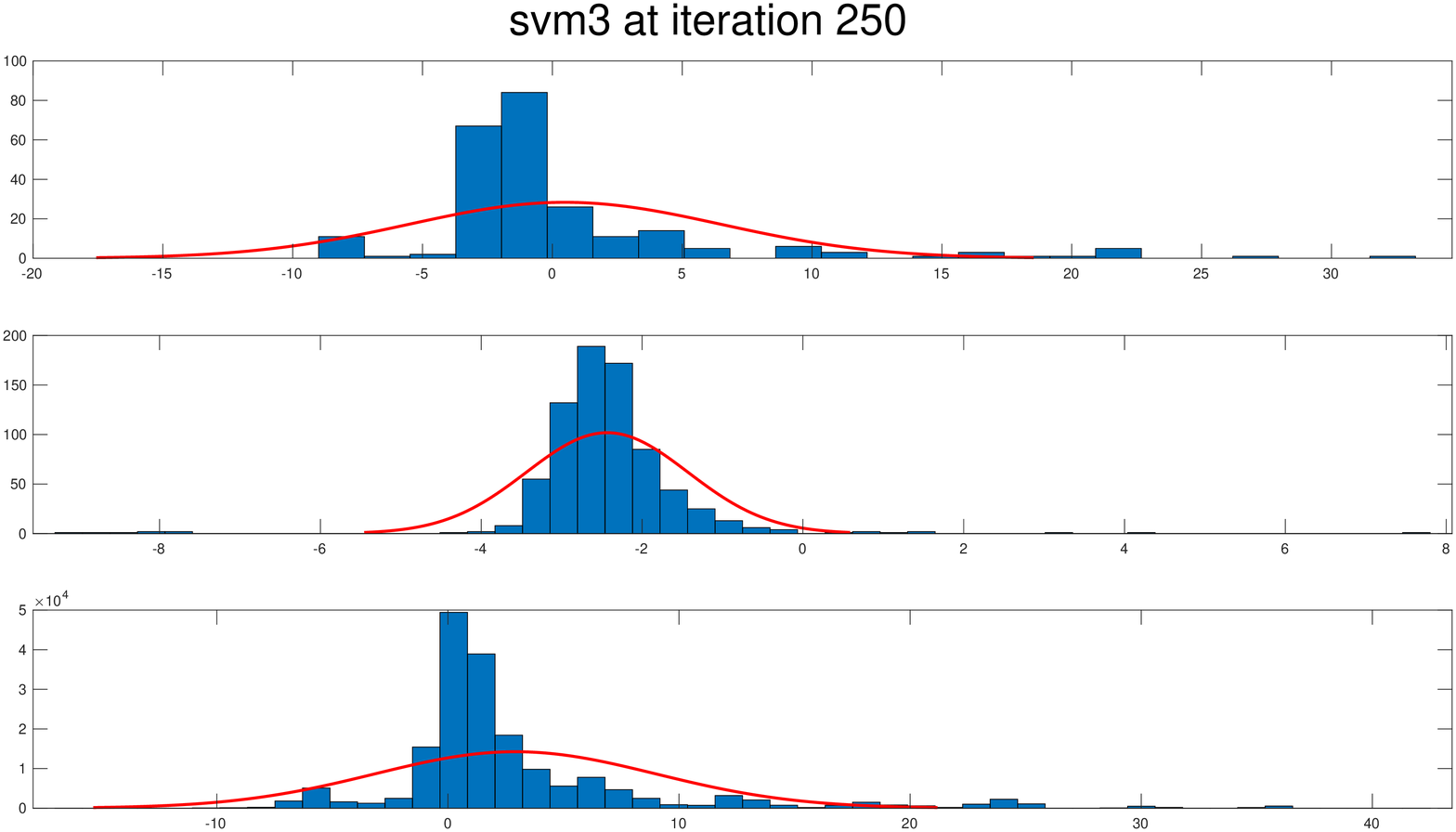}
\end{subfigure} %
\begin{subfigure}[t]{0.47\textwidth}
  \centering
  \includegraphics[width=\linewidth]{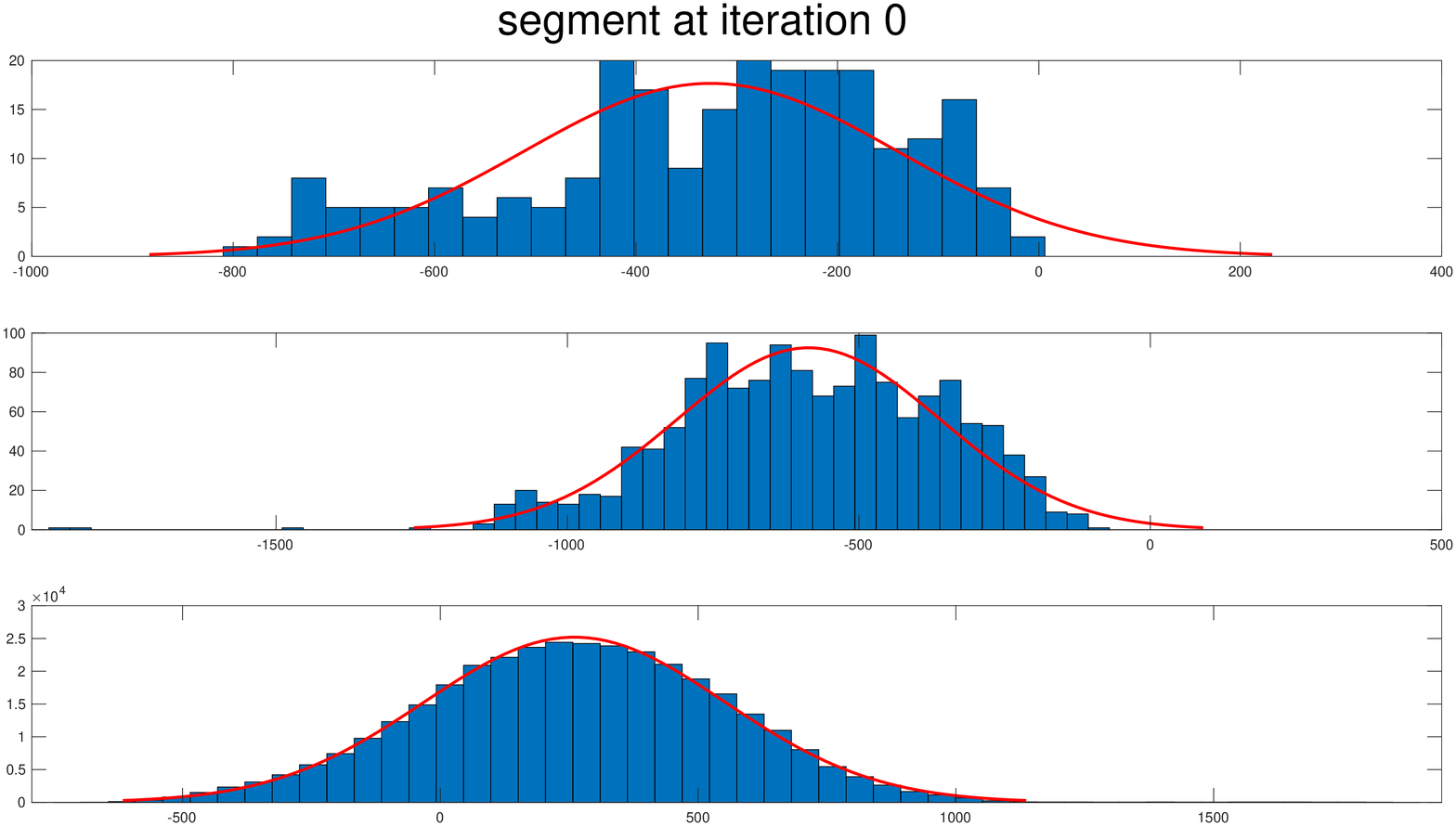}
  \end{subfigure}%
  \hfill
\begin{subfigure}[t]{0.47\textwidth}
  \centering
  \includegraphics[width=\linewidth]{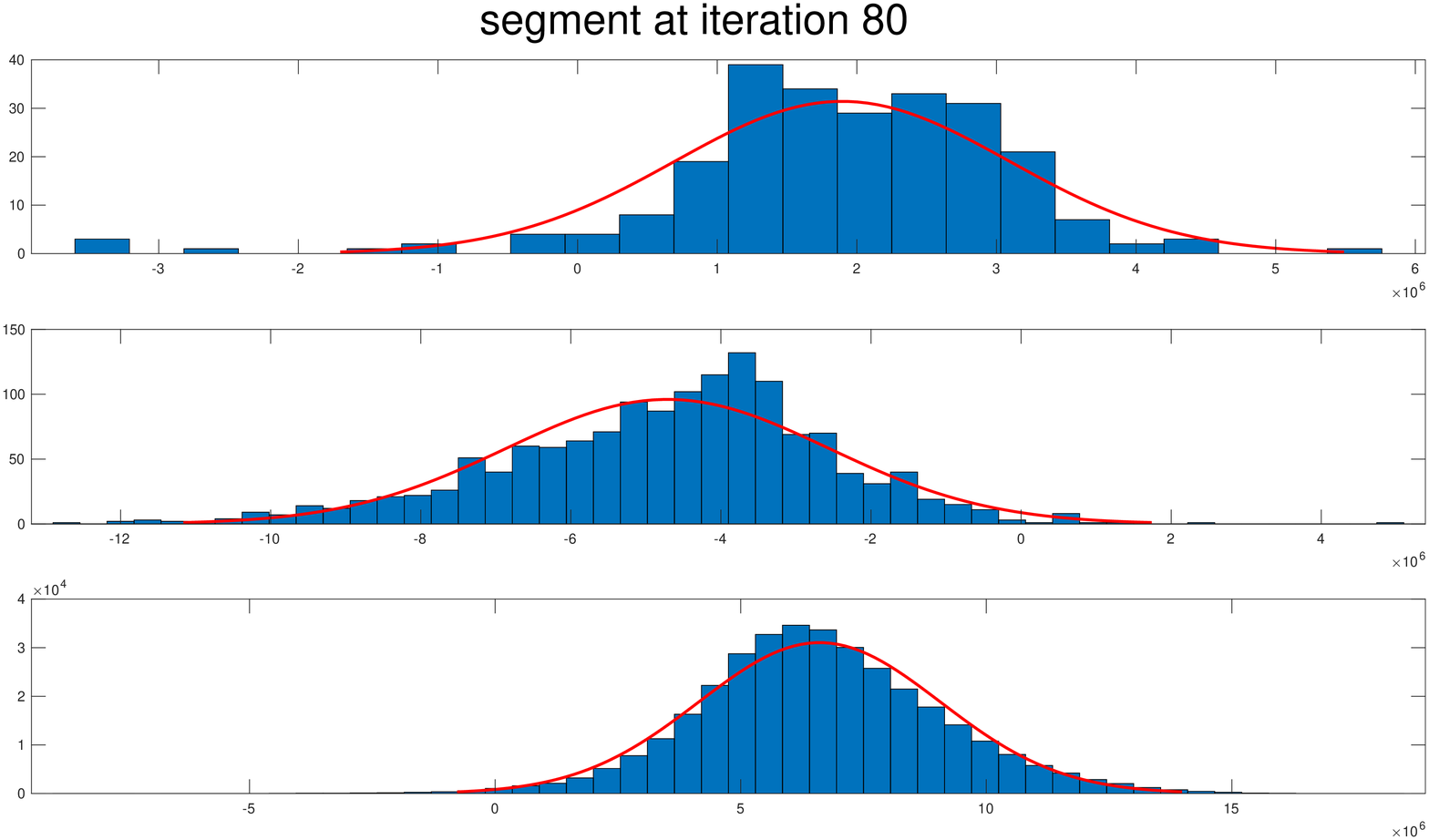}
\end{subfigure} %
\caption{Distribution for $w^\top X^+$, $w^\top X^-$ and $w^\top (X^+-X^-)$ at different iterations during training for {diabetes}.}
\label{fig:normal1}
\end{figure}

%\begin{figure}[tph] 
%\begin{subfigure}
%  \centering
%  \includegraphics[height=0.50\textwidth,width=0.48\linewidth]{images/a9a5.png}
%  \end{subfigure}%
%\begin{subfigure}
%  \centering
%  \includegraphics[height=0.50\textwidth,width=0.48\linewidth]{images/a9a6.png}
%\end{subfigure} %
%\caption{Nearly normal distribution for $w^\top X^+$, $w^\top X^-$ and $w^\top (X^+-X^-)$ at different iterations during training.}
%\label{fig:normal}
%\end{figure}

We next demonstrate that minimizing our proposed models $F_{n01}$ and $F_{nrank}$ results in good classifiers and efficient training methods. 

{\bf Method} Based on the results of the previous section, we propose the following method of training a linear classifier for a given training set. 
\begin{enumerate}
\item Using the positive and negative samples of the training set, compute the empirical estimates of $\mu^+$,  $\mu^-$, as well as $\Sigma^+$, $\Sigma^-$ for the case of $F_{n01}$ and   $\Sigma^{++}$, $\Sigma^{+-}$,  $\Sigma^{+-}$, and $\Sigma^{--}$
for the case of $F_{nrank}$. Note that number of samples for positive and negative parts are not balanced and it is not clear how to get the empirical estimations of $\Sigma^{+-}$ and $\Sigma^{--}$. Theoretically, the random vectors with positive and negative labels should be uncorrelated and thus $\Sigma^{+-}$ and $\Sigma^{+-}$ are zero matrices. If we use some sampling methods to get the estimations for $\Sigma^{+-}$ and $\Sigma^{+-}$, the actual performance is worse than taking them as zero matrices.
\item By setting $P({Y}=1)=n_+/n$ and $P({Y}=-1)=n_-/n$, $F_{n01}(w)$ is defined as in    \eqref{smooth.zero.one} and $\nabla F_{n01}(w)$ is defined as in \eqref{eq:col.error.short}.  Similarly, $F_{nrank}(w)$ is defined as in \eqref{F.AUC.smooth}  and $\nabla F_{nrank}(w)$ is defined as in \eqref{eq:gradrank}. 
\item For minimizing the prediction error,  apply L-BFGS method with Wolfe  line-search until $w_\epsilon$ is reached such that $\|\nabla  F^\lambda_{n01}(w) \|\leq \epsilon$, where $ F^\lambda_{n01}(w)=F_{n01}(w)+\lambda (1-\|w\|^2)^2$, for given a tolerance $\epsilon$. For ranking loss minimization apply the same to $ F^\lambda_{nrank}(w)=F_{nrank}(w)+\lambda (1-\|w\|^2)^2$. 
Return $w_\epsilon$ as the linear classifier. 
\end{enumerate}

{\bf Computational cost} We compare the performance of the classifiers obtained by the proposed method  to the state-of-the art linear classification methods. In particular, we compare optimizing $F^\lambda_{n01}$  vs.  
  regularized logistic regression, and optimizing $F^\lambda_{nrank}$  vs. regularized pairwise hinge loss. The pairwise hinge loss is chosen  as the most efficient surrogate for the ranking loss, because the cost of gradient computation is roughly ${\cal O}(n\log n)$, instead of ${\cal O}(n^+n^-)$.
   For logistic regression we optimize function $ \hat F^\lambda_{log}(w)= \hat F^\lambda_{log}(w)+\lambda \|w\|^2$
   and for pairwise hinge loss  we  optimized $\hat F^\lambda_{hinge}(w)= \hat F^\lambda_{hinge}(w)+\lambda \|w\|^2$. 
   We use L-BFGS method with Wolfe  line-search for all  four functions,  $F^\lambda_{n01}$, $F^\lambda _{nrank}$ , $ \hat F^\lambda_{log}(w)$ and $\hat F^\lambda_{hinge}(w)$.    
   Even though
   $\hat F^\lambda_{hinge}(w)$ is not smooth, it is known that  L-BFGS works very well  for such functions \cite{overton_BFGS,Jin_BFGS}. 
   Note that each gradient computation for  $ \hat F^\lambda_{log}(w)$ and $\hat F^\lambda_{hinge}(w)$  require 
   ${\cal O}(nd)$ and ${\cal O}(dn\log n)$ operations respectively (in the case of sparse data, the dependence on $d$ 
   reduces according to sparsity). On the other hand applying the same method to $F^\lambda_{n01}$ and $F^\lambda _{nrank}$ requires only 
${\cal O}(d^2)$   operations, (and when the data is sparse, the dependence on $d$ reduces to as low as ${\cal O}(d)$, depending on the sparsity of the covariance matrices). The covariance matrix computation requires ${\cal O}(nd^2)$ operations  (${\cal O}(nd)$ in the sparse case), however, this computation is done once before the optimization algorithm. 
For the problems with large number of sparse features, such as {\em rcv1} and {\em realsim} listed below, 
We can use the empirical covariance estimation in the form $\Sigma=\frac{1}{n-1}X^TX-\frac{n}{n-1}\bar{X}\bar{X}^T$ where $X$ is taken as the data matrix here and $\bar{X}$ is the mean vector of $X$. It is efficient in terms of both memory storage and computational cost since we don't actually have to compute or store $\Sigma$ explicitly.

 {\bf Alternative methods}  Note that function $F^\lambda_{n01}$ $F^\lambda_{nrank}$  have well defined Hessians as well, hence second order methods can be applied to minimize these functions. 
   We have used preconditioned conjugate gradient method and
   other  second order methods  based on  Hessian vector products, however, they did not outperform the L-BFGS in terms of time, while achieving similar accuracy.

 {\bf Starting point} $ F^\lambda_{n01}(w)$ and $ F^\lambda_{nrank}$ are nonconvex functions, thus the results of our optimization approach may depend on the starting point. In our experiments we  used the  following starting point
 %the solution of the following problem as a warm starting point
%\begin{equation*}
%\begin{aligned}
%\min_{w_0 \in \mathbb{R}^d} ~ \frac{1}{2} \|\mu^+ - w_0\|~~~s.t.~~{w_0}^T \mu^- = 0
%\end{aligned}
%\end{equation*}
%which is obtained  as follows,
\begin{equation*}
w_0=\frac{\bar w_0}{\|\bar w_0\|}, \quad \text{where}\quad \bar w_0 = \mu^+ - \frac{{\mu^-}^T \mu^+}{\|\mu^-\|^2} \mu^-.
\end{equation*}
We have also tried random starting points, but the results were not better than  using  $w_0$ defined above. 
For logistic loss and hinge loss, we simply generate random starting points, since these functions are convex.

  {\bf Details}
 All experiments were run using Python3.5 on a Win10 with 3.60 GHz Intel Core i7 processor and 8GB of RAM. All functions were trained using L-BFGS with memory size set to $20$ and the Wolfe line-search parameters were set as $c1=0.0001, c2=0.9$. The maximum number of iterations was set to 500 and the first order optimal threshold was chosen as $10^{-4}$.

{\bf Artificial data} For our first set of experiments, we have generated 9 different artificial Gaussian data sets of various dimensions using random first and second moments;  they are summarized in Table \ref{t1_data_art}. Moreover, for each set we   generated  some percentage of outliers by swapping  labels of positive and negative examples in the training data. 
We set the regularization parameter to zero for the experiments with artificial data sets.

% Table 1
\begin{table}[tph]  
\centering
\caption{Artificial data sets statistics. $d:$ number of features, $n:$ number of data points, $P^+, P^-:$ prior probabilities, $out:$ percentage of outlier data points.}
\label{t1_data_art}
\begin{center}
\begin{small}
\resizebox{0.50\columnwidth}{!}{
 \begin{tabular}{ccccccr} 
%\abovespace\belowspace
Name&$d$&${n}$&$P^+$&$P^-$&$out \%$\\ 
\hline
%\abovespace
$data_1$&500&5000&0.05&0.95&0\\
$data_2$&500&5000&0.35&0.65&5\\
$data_3$&500&5000&0.5&0.5&10\\
$data_4$&1000&5000&0.15&0.85&0\\
$data_5$&1000&5000&0.4&0.6&5\\
$data_6$&1000&5000&0.5&0.5&10\\
$data_7$&2500&5000&0.1&0.9&0\\
$data_8$&2500&5000&0.35&0.65&5\\
$data_9$&2500&5000&0.5&0.5&10\\
\end{tabular}
}
\end{small}
\end{center}
\vskip -0.1in
\end{table}

The corresponding numerical results for $F_{n01}(w)$ are summarized in Table \ref{t_art_0-1_2}, where we used 80 percent of the data points as the training data and the rest  as the test data. The reported average accuracy is based on 20 runs for each data set.  When minimizing $F_{n01}(w)$, we  used the exact moments from which the data set was  generated, and also the approximate moments, empirically obtained from the training data. The bold numbers indicate the 
average testing accuracy attained by  minimizing $F_{n01}(w)$ using approximate moments, when this accuracy is significantly better than that obtained by minimizing $F_{log}(w)$. 

We see in Table \ref{t_art_0-1_2} that, as expected, minimizing $F_{n01}(w)$ using the exact moments produces linear classifiers with superior performance overall, while  minimizing $F_{n01}(w)$ using approximate moments outperforms  minimizing $F_{log}(w)$, except for {\em data8} and {\em data9} where the number of data points is small compared to dimension and the moments estimates are not accurate. Note also that minimizing $F_{n01}(w)$ requires less time than minimizing $F_{log}(w)$.  

% Table 2
\begin{table}[tph]   
\centering
 \caption{$F_{n01}(w)$ vs. $F_{log}(w)$ minimization on artificial data sets.}
   \vskip 0.1in
  \label{t_art_0-1_2}
   \vskip 0.05in
\begin{center}
\begin{small}
 \begin{tabular}{ c|cc|cc|cc}  
  %\abovespace\belowspace
\multirow{3}{*}{Data} 
&\multicolumn{2}{c|}{\textbf{$\pmb{F_{n01}(w)}$ Minimization}}&\multicolumn{2}{c|}{\textbf{$\pmb{F_{n01}(w)}$ Minimization}}&\multicolumn{2}{c}{\textbf{$\pmb{F_{log}(w)}$ Minimization}}\\  
&\multicolumn{2}{c|}{\textbf{Exact moments}}&\multicolumn{2}{c|}{\textbf{Approximate moments}}&&\\ 
&{Accuracy$\pm$ std}&{Time (s)}&{Accuracy $\pm$ std}&{Time (s)}&{Accuracy $\pm$ std}&{Time (s)}\\
\hline
%\abovespace
$data_1$&0.9965$\pm$0.0008&0.25&0.9907$\pm$0.0014&1.04&0.9897$\pm$0.0018&3.86\\ 
$data_2$&0.9905$\pm$0.0023&0.26&\textbf{0.9806}$\pm$0.0032&0.86&0.9557$\pm$0.0049&13.72\\
$data_3$&0.9884$\pm$0.0030&0.03&\textbf{0.9745}$\pm$0.0037&1.28&0.9537$\pm$0.0048&15.79\\ 
$data_4$&0.9935$\pm$0.0017&0.63&0.9791$\pm$0.0034&5.51&0.9782$\pm$0.0031&10.03\\
$data_5$&0.9899$\pm$0.0026&5.68&\textbf{0.9716}$\pm$0.0048&10.86&0.9424$\pm$0.0055&28.29\\
$data_6$&0.9904$\pm$0.0017&0.83&\textbf{0.9670}$\pm$0.0058&5.18&0.9291$\pm$0.0076&25.47\\
$data_7$&0.9945$\pm$0.0019&4.79&0.9786$\pm$0.0028&32.75&0.9697$\pm$0.0031&43.20\\ 
$data_8$&0.9901$\pm$0.0013&9.96&0.9290$\pm$0.0045&119.64&0.9263$\pm$0.0069&104.94\\
$data_9$&0.9899$\pm$0.0028&1.02&0.9249$\pm$0.0096&68.91&0.9264$\pm$0.0067&123.85\\
\end{tabular}
\end{small}
\end{center}
\vskip -0.1in
\end{table}
% \begin{table}[tph]   
% \centering
%  \caption{$F_{n01}(w)$ vs. $\hat F_{log}(w)$ minimization on artificial data sets.}
%    \vskip 0.1in
%   \label{t_art_0-1}
%    \vskip 0.05in
% \begin{center}
% \begin{small}
% \resizebox{0.7\columnwidth}{!}{
%  \begin{tabular}{ p{0.6cm}|p{0.9cm}p{0.8cm}|p{0.9cm}p{0.8cm}|p{0.9cm}p{0.6cm}}  
%   %\abovespace\belowspace
% \multirow{4}{*}{Data} 
% &\multicolumn{2}{c|}{\textbf{$\pmb{F_{n01}(w)}$  }}&\multicolumn{2}{c|}{\textbf{$\pmb{F_{n01}(w)}$  }}&\multicolumn{2}{c}{\textbf{$\pmb{F_{log}(w)}$ }}\\  
% &\multicolumn{2}{c|}{\textbf{Exact}}&\multicolumn{2}{c|}{\textbf{Approximate}}&&\\ 
% &\multicolumn{2}{c|}{\textbf{moments}}&\multicolumn{2}{c|}{\textbf{moments}}&&\\ 
% &{Accuracy}&{Time(s)}&{Accuracy}&{Time(s)}&{Accuracy}&{Time(s)}\\
% \hline
% %\abovespace
% $data_1$&0.9965&0.25&0.9907&1.04&0.9897&3.86\\ 
% $data_2$&0.9905&0.26&\textbf{0.9806}&0.86&0.9557&13.72\\
% $data_3$&0.9884&0.03&\textbf{0.9745}&1.28&0.9537&15.79\\ 
% $data_4$&0.9935&0.63&0.9791&5.51&0.9782&10.03\\
% $data_5$&0.9899&5.68&\textbf{0.9716}&10.86&0.9424&28.29\\
% $data_6$&0.9904&0.83&\textbf{0.9670}&5.18&0.9291&25.47\\
% $data_7$&0.9945&4.79&0.9786&32.75&0.9697&43.20\\ 
% $data_8$&0.9901&9.96&0.9290&119.64&0.9263&104.94\\
% $data_9$&0.9899&1.02&0.9249&68.91&0.9264&123.85\\
% \end{tabular}
% }
% \end{small}
% \end{center}
% \vskip -0.1in
% \end{table}

In Table \ref{t_art_AUC_2} we compare the performance of  linear classifiers obtained by optimizing $F_{nrank}(w)$ defined in \eqref{F.AUC.smooth} and the pairwise hinge loss, $\hat F(w) = F_{hinge}(w)$ as is defined in \eqref{AUC_hinge}
on the artificial data described in Tables \ref{t1_data_art}. 

The  results  are summarized as in Table \ref{t_art_0-1_2}, except that we report the AUC value as the performance measure. As we can see in Table \ref{t_art_AUC_2}, the performance of the linear classifier obtained through minimizing $F_{nrank}(w)$ using approximate moments surpasses that of the classifier obtained via minimizing $\hat F_{hinge}(w)$, both in terms of the average AUC value as well as the required solution time.

%\newpage 
% Table 3
\begin{table}[tph]     
\centering 
 \caption{$F_{nrank}(w)$ vs. $F_{hinge}(w)$ minimization  on artificial data sets.}
   \vskip 0.1in
  \label{t_art_AUC_2}
   \vskip 0.05in
\begin{center}
\begin{small}
 \begin{tabular}{ c|cc|cc|cc}  
 % \abovespace\belowspace
\multirow{3}{*}{Data} 
&\multicolumn{2}{c|}{\textbf{$\pmb{F_{AUC}(w)}$ Minimization}}&\multicolumn{2}{c|}{\textbf{$\pmb{F_{AUC}(w)}$ Minimization}}&\multicolumn{2}{c}{\textbf{$\pmb{F_{hinge}(w)}$ Minimization}}\\  
&\multicolumn{2}{c|}{\textbf{Exact moments}}&\multicolumn{2}{c|}{\textbf{Approximate moments}}&\\ 
&{AUC$\pm$ std}&{Time (s)}&{AUC $\pm$ std}&{Time (s)}&{AUC $\pm$ std}&{Time (s)}\\  
\hline
%\abovespace
$data_1$&0.9972$\pm$0.0014&0.01&\textbf{0.9941}$\pm$0.0027&0.23&0.9790$\pm$0.0089&5.39\\
$data_2$&0.9963$\pm$0.0016&0.01&\textbf{0.9956}$\pm$0.0018&0.22&0.9634$\pm$0.0056&159.23\\
$data_3$&0.9965$\pm$0.0015&0.01&\textbf{0.9959}$\pm$0.0018&0.24&0.9766$\pm$0.0041&317.44\\
$data_4$&0.9957$\pm$0.0018&0.02&\textbf{0.9933}$\pm$0.0022&0.83&0.9782$\pm$0.0054&23.36\\
$data_5$&0.9962$\pm$0.0011&0.02&\textbf{0.9951}$\pm$0.0013&0.80&$0.9589\pm$0.0068&110.26\\
$data_6$&0.9962$\pm$0.0013&0.02&\textbf{0.9949}$\pm$0.0015&0.82&0.9470$\pm$0.0086&275.06\\
$data_7$&0.9965$\pm$0.0021&0.08&\textbf{0.9874}$\pm$0.0034&4.61&0.9587$\pm$0.0092&28.31\\ 
$data_8$&0.9966$\pm$0.0008&0.07&\textbf{0.9929}$\pm$0.0017&4.54&0.9514$\pm$0.0051&104.16\\
$data_9$&0.9962$\pm$0.0014&0.08&\textbf{0.9932}$\pm$0.0020&4.54&0.9463$\pm$0.0085&157.62\\
\end{tabular}
\end{small}
\end{center}
\vskip -0.1in
\end{table}
% \begin{table}[tph]     
% \centering 
%  \caption{$F_{nrank}(w)$ vs. $\hat F_{hinge}(w)$ minimization on artificial data sets.}
%    \vskip 0.1in
%   \label{t_art_AUC}
%    \vskip 0.05in
% \begin{center}
% \begin{small}
% \resizebox{0.7\columnwidth}{!}{
%  \begin{tabular}{ p{0.6cm}|p{0.7cm}p{0.8cm}|p{0.7cm}p{0.9cm}|p{0.7cm}p{0.6cm}}  
%  % \abovespace\belowspace
% \multirow{4}{*}{Data} 
% &\multicolumn{2}{c|}{\textbf{$\pmb{F_{nrank}(w)}$ }}&\multicolumn{2}{c|}{\textbf{$\pmb{F_{AUC}(w)}$}}&\multicolumn{2}{c}{\textbf{$\pmb{F_{hinge}(w)}$}}\\  
% &\multicolumn{2}{c|}{\textbf{Exact}}&\multicolumn{2}{c|}{\textbf{Approximate}}&\\ 
% &\multicolumn{2}{c|}{\textbf{moments}}&\multicolumn{2}{c|}{\textbf{moments}}&&\\ 
% &{AUC}&{Time(s)}&{AUC}&{Time(s)}&{AUC}&{Time(s)}\\  
% \hline
% %\abovespace
% $data_1$&0.9972&0.01&\textbf{0.9941}&0.23&0.9790&5.39\\
% $data_2$&0.9963&0.01&\textbf{0.9956}&0.22&0.9634&159.23\\
% $data_3$&0.9965&0.01&\textbf{0.9959}&0.24&0.9766&317.44\\
% $data_4$&0.9957&0.02&\textbf{0.9933}&0.83&0.9782&23.36\\
% $data_5$&0.9962&0.02&\textbf{0.9951}&0.80&0.9589&110.26\\
% $data_6$&0.9962&0.02&\textbf{0.9949}&0.82&0.9470&275.06\\
% $data_7$&0.9965&0.08&\textbf{0.9874}&4.61&0.9587&28.31\\ 
% $data_8$&0.9966&0.07&\textbf{0.9929}&4.54&0.9514&104.16\\
% $data_9$&0.9962&0.08&\textbf{0.9932}&4.54&0.9463&157.62\\
% \end{tabular}
% }
% \end{small}
% \end{center}
% \vskip -0.1in
% \end{table}

{\bf Real data} We now compare the performance of $F_{n01}$ vs. logistic regression and  $F_{nrank}$ vs. pairwise hinge loss on 21  data sets downloaded from LIBSVM website\footnote{\url{https://www.csie.ntu.edu.tw/~cjlin/libsvmtools/datasets/binary.html}} and UCI machine learning repository\footnote{\url{http://archive.ics.uci.edu/ml/}}, summarized in Table \ref{t1}. The data sets from UCI machine learning repository with categorical features are transformed into grouped binary features.
We have normalized  the data sets so that each feature does not exceed $1$ in absolute value.

 % Table 4
\begin{table}[tph]  
\centering
\caption{Data sets statistics. $d:$ number of features, $n:$ number of data points, $P^+, P^-:$ prior probabilities, \\$AC:$ attribute characteristics.}
\label{t1}
\begin{center}
\begin{small}
\resizebox{0.5\columnwidth}{!}{
 \begin{tabular}{ccccccr} 
%\abovespace\belowspace
Name&AC&$d$&${n}$&$P^+$&$P^-$\\ 
\hline
%\abovespace
fourclass&real&2&862&0.35&0.65\\
svm1&real&4&3089&0.35&0.65\\
diabetes&real&8&768&0.35&0.65\\
%shuttle&real&9&43500&0.22&0.78\\
vowel&int&10&528&0.09&0.91\\
magic04&real&10&19020&0.35&0.65\\
poker&int&11&25010&0.02&0.98\\
letter&int&16&20000&0.04&0.96\\ 
segment&real&19&210&0.14&0.86\\
svm3&real&22&1243&0.23&0.77\\
ijcnn1&real&22&35000&0.1&0.9\\
german&real&24&1000&0.3&0.7\\
landsat satellite&int&36&4435&0.09&0.91\\
sonar&real&60&208&0.5&0.5\\
a9a&binary&123&32561&0.24&0.76\\
%connect-4&binary&126&\textcolor{red}{10000}&\textcolor{red}{NN}&0.71&0.29\\
w8a&binary&300&49749&0.02&0.98\\
mnist&real&782&100000&0.1&0.9\\
%epsilon&$[-1,1]$, real&2000&400000&100000 &&\\
colon-cancer&real&2000&62&0.35&0.65\\
gisette&real&5000&6000&0.49&0.51\\
covtype &binary &54 &581012 & 0.49&0.51 \\
rcv1&real&47236&20242&0.52&0.48\\
real-sim& real & 20958 & 72309 & 0.31 & 0.69\\
\end{tabular}
}
\end{small}
\end{center}
\vskip -0.1in
\end{table}

We used five-fold cross-validation using random data split and repeated each experiment four times, the results reported are averaged over 20 runs. The regularization parameter $\lambda$ for $F^\lambda_{n01}$  has been set to $0.001$ as this has been observed to be a good fixed value, while for   $\hat F^\lambda_{log}$ it has been set to $1/n$, which is often suggested in the literature. Full tuning of $\lambda$ for both models can be performed, however, the effect of different $\lambda$ on minimization of $F^\lambda_{n01}$ is somewhat different from the usual regularization, since the function and the regularizer is not convex and local minima may be observed. On the other hand, tuning $\lambda$ for logistic regression is computationally costly. 
%Exploring the effect of $\lambda$ and the need to tune it is subject for future study.  

In Table \ref{T_Error_logloss_real} we see the comparison of the average testing accuracy of the resulting linear classifier as well as the average number of iterations performed by the algorithms and the average CPU time. We can see that in almost all cases the testing accuracy achieved by both methods is very similar, with a few cases when one approach dominates the other.  However, the solution time of our method is often much smaller, especially on large instances. 

% Table 5
 \begin{table}[tph]
\caption {Comparison of minimizing $F_{n01}(w)$ vs. $F_{log}(w)$.}
\label{T_Error_logloss_real}
  \centering
  \begin{small}
  \resizebox{\columnwidth}{!}{
  \begin{tabular}{c|cccc|ccc}
   % \hline 
    & \multicolumn{4}{c|}{$\pmb{F_{n01}}$} & \multicolumn{3}{c}{$\pmb{F_{log}}$}\\
    Data & accuracy&num. iters&sol time&moment time& accuracy&num. iters&sol time\\
    \hline
    fourclass & 0.7564$\pm$ 0.0323&10.35 &0.01$\pm$ 0.00& 0.00& 0.7602$\pm$ 0.0293 &7.45 &0.01$\pm$ 0.00\\
    svm1 & \textbf{0.9455$\pm$ 0.0092} &23.75 & 0.03$\pm$ 0.00&0.00 &0.9306$\pm$ 0.0131 &16.00& 0.02$\pm$ 0.00 \\
    diabetes & 0.7667$\pm$ 0.0371& 25.45 & 0.03$\pm$ 0.01  &0.00 &0.7680$\pm$ 0.0397 &18.95  &0.01$\pm$ 0.00 \\
    vowel &0.9619$\pm$ 0.0207 &36.60& 0.05$\pm$ 0.00& 0.00 & 0.9652$\pm$ 0.0176 &18.70 & 0.01$\pm$ 0.00\\
    magic &0.7665$\pm$ 0.0091&43.05 &  0.06$\pm$ 0.00 &0.00 &\textbf{0.7897$\pm$ 0.0087} &25.55 &0.04$\pm$ 0.01\\
    poker& 0.9795$\pm$ 0.0017& 16.10 & 0.02$\pm$ 0.00& 0.00& 0.9795$\pm$ 0.0017& 30.75& 0.07$\pm$ 0.01 \\
    letter & 0.9710$\pm$ 0.0030  &85.65&  0.13$\pm$ 0.01& 0.00 &\textbf{0.9824$\pm$ 0.0019}& 67.10 & 0.12$\pm$ 0.02\\
    segment &0.9845$\pm$ 0.0292& 401.25 &0.65$\pm$ 0.22& 0.00& \textbf{0.9978$\pm$ 0.0022} &97.85 &0.11$\pm$ 0.01\\
    svm3  &\textbf{0.8208$\pm$ 0.0245}& 214.25 &0.33$\pm$ 0.04 &0.00 &0.7929$\pm$ 0.0209 &18.90 &0.02$\pm$ 0.00\\
    ijcnn1 &0.9054$\pm$ 0.0026 &41.9 &0.07$\pm$ 0.01&0.00 &0.9142$\pm$ 0.0025& 32.00 &0.10$\pm$ 0.02 \\
    german & 0.7553$\pm$ 0.0252& 35.30&0.05$\pm$ 0.00 &0.00&0.7648$\pm$ 0.0320 &25.60 &0.02$\pm$ 0.00 \\
    satimage  &0.9064$\pm$ 0.0064 &13.00 & 0.02$\pm$ 0.00 & 0.00 &0.9068$\pm$ 0.0060 &490.75& 0.70$\pm$ 0.03\\
    sonar& 0.7573$\pm$ 0.0610& 500.00& 0.66$\pm$ 0.01 &0.00 &0.7549$\pm$ 0.0761& 14.90 &0.01$\pm$ 0.00 \\
    a9a &0.8376$\pm$ 0.0043&130.10 &0.35$\pm$ 0.04 &0.02& 0.8472$\pm$ 0.0041 &75.85 &1.27$\pm$ 0.09\\
    w8a &0.9807$\pm$ 0.0013 & 273.95 &2.00$\pm$ 0.12& 0.07 &0.9842$\pm$ 0.0012&24.90  &1.60$\pm$ 0.16 \\
    mnist &0.9819$\pm$ 0.0008& 500.00 &16.49$\pm$ 0.21 &0.54&0.9877$\pm$ 0.0005&112.55 &36.88$\pm$ 2.21 \\
    colon &\textbf{0.7833$\pm$ 0.1191} &17.50  &0.48$\pm$ 0.06 &0.04& 0.7167$\pm$ 0.1221& 54.15 &0.11$\pm$ 0.01\\
    gisette &0.9753$\pm$ 0.0035&83.30& 14.60$\pm$ 2.38& 1.06 &0.9714 $\pm$ 0.0043 &156.10 & 21.64$\pm$ 1.79\\
    covtype& 0.5502$\pm$ 0.0134 & 500.00& 7.93$\pm$ 0.13 &0.11 &\textbf{0.7562$\pm$0.0010}& 97.50 &15.87$\pm$2.35 \\
    rcv1 & 0.9632$\pm$ 0.0026 &73.35 &26.54$\pm$ 2.32 &1.37 & 0.9595$\pm$ 0.0024& 15.30 &56.37$\pm$ 1.95\\
    realsim & 0.9547$\pm$ 0.0018 & 500.00 &263.84$\pm$10.19 & 2.67 & \textbf{0.9676$\pm$0.0014}  &16.65 &1367.80$\pm$62.33 \\
% \hline
  \end{tabular} 
  }
  \end{small}
\end{table}
\begin{table}[tph]
\caption {Comparison of minimizing $F_{nrank}(w)$ vs. $F_{hinge}(w)$.}
\label{T_AUC_hinge_real}
  \centering
  \resizebox{\columnwidth}{!}{
  \begin{tabular}{c|cccc|cccc}
   % \hline 
    & \multicolumn{4}{c|}{$\pmb{F_{nrank}}$} & \multicolumn{3}{c}{$\pmb{F_{hinge}}$}\\
    Data & accuracy&num. iters&sol time&moment time& accuracy&num. iters&sol time\\
    \hline
fourclass &0.8362$\pm$ 0.0312 &7.00 & 0.01$\pm$ 0.00  &0.00 &0.8361$\pm$ 0.0312 &11.95 &0.15$\pm$ 0.01\\
svm1 & 0.9717$\pm$ 0.0065  &13.20& 0.01$\pm$ 0.00 &0.00 &\textbf{0.9841$\pm$ 0.0041} &11.95 &0.53$\pm$ 0.04 \\
diabetes & 0.8311$\pm$ 0.0312 &14.65& 0.01$\pm$ 0.00 & 0.00 &0.8308$\pm$ 0.0329 & 20.80  &0.29$\pm$ 0.25\\
shuttle &0.9840$\pm$ 0.0016 &63.90 & 0.07$\pm$ 0.01 & 0.00& 0.9892$\pm$ 0.0015 &12.85 &8.66$\pm$ 0.48\\
vowel& 0.9585$\pm$ 0.0333& 19.30 & 0.02$\pm$ 0.00 &0.00 &\textbf{0.9737$\pm$ 0.0202} &36.35 &0.34$\pm$0.17 \\
magic& 0.8382$\pm$ 0.0071& 22.30 & 0.02$\pm$ 0.00& 0.00 &0.8428$\pm$ 0.0070 &20.25 &5.74$\pm$ 0.49 \\
poker& 0.5053$\pm$ 0.0224 &15.75& 0.02$\pm$ 0.00 &0.00& 0.5070$\pm$ 0.0223& 28.80  &11.26$\pm$ 2.55 \\
letter  &0.9830$\pm$ 0.0029&  23.45 &0.02$\pm$ 0.00 &0.00 &0.9884$\pm$ 0.0022 &31.25 &9.00$\pm$ 1.36 \\
segment &0.9947$\pm$ 0.0055 &261.30 &0.27$\pm$ 0.15 &0.00& 0.9999$\pm$ 0.0001 &37.60 &1.28$\pm$0.22 \\
svm3 & \textbf{0.7996$\pm$ 0.0421 }&115.95 & 0.67$\pm$ 0.08 & 0.00& 0.7731$\pm$ 0.0457& 25.55 &0.48$\pm$ 0.08\\
ijcnn1 &0.9269$\pm$ 0.0036 &31.00 & 0.03$\pm$ 0.00&0.00& 0.9291$\pm$ 0.0037 &35.55&19.48$\pm$ 0.99\\
german  &0.7938$\pm$ 0.0292& 26.90& 0.03$\pm$ 0.00& 0.00& 0.7929$\pm$ 0.0292& 35.60 &0.56$\pm$ 0.08 \\
satimage & 0.7561$\pm$ 0.0163 &80.00 & 0.09$\pm$ 0.02 &0.00 &0.7665$\pm$ 0.0193 & 78.00 &5.33$\pm$ 1.36 \\
sonar& 0.8150$\pm$ 0.0672& 500.00 &0.51$\pm$ 0.01 &0.00 & \textbf{0.8470$\pm$ 0.0559}& 113.90 &0.77$\pm$ 1.17 \\
a9a &0.9002$\pm$ 0.0040& 205.90 &0.24$\pm$ 0.04& 0.02 &0.9033$\pm$ 0.0037 &81.15 &48.04$\pm$2.50\\
w8a &0.9631$\pm$ 0.0058 &422.75 &0.58 $\pm$ 0.05 & 0.07 & 0.9659$\pm$ 0.0049& 400.05 &606.08$\pm$ 137.09 \\
mnist &0.9942$\pm$ 0.0009 &500.00 &0.88$\pm$ 0.07  &0.55&0.9953 $\pm$ 0.0007 &70.80 &516.67$\pm$ 21.13 \\
colon &0.8715$\pm$ 0.0933 &13.40 &0.16$\pm$ 0.01&0.04 &0.8774$\pm$ 0.0998 &78.35& 0.58$\pm$ 0.20\\
gisette &0.9962$\pm$ 0.0012& 20.60& 1.61$\pm$ 0.08 & 1.05 &0.9943$\pm$ 0.0013& 73.95 &107.49$\pm$ 23.64 \\
covtype &0.8243$\pm$ 0.0001& 240.20& 0.24$\pm$ 0.02 & 0.11 &0.8272$\pm$ 0.0009& 192.65 &1872.67$\pm$ 104.84 \\
rcv1  &0.9934$\pm$ 0.0008 &15.90 &5.38$\pm$ 0.32 &1.42 &0.9941$\pm$0.0008 &23.50 &3712.87$\pm$ 229.71 \\
realsim  &0.9916$\pm$ 0.006 &46.20 &20.16$\pm$ 1.22 &2.82 & & & \\
    %\hline
  \end{tabular} 
  }
\end{table}

In Table \ref{T_AUC_hinge_real} we compare the average testing AUC of two linear classifiers as well as the average number of iterations performed by the algorithms and the average CPU time. For these experiments we set $\lambda$ to $0.001$ for $F^\lambda_{nrank}$, but for $\hat F^\lambda_{hinge}$ we set it to $1/\sqrt{n^+n^-}$, to mimic the choice of the regularization term in the case of logistic regression. 
%since this resulted in superior performance\footnote{Since $\hat F_{hinge}$ scales as $1/n^+\cdot n^-$ then only very small values of $\lambda$ are apprppriate for regularization in order to balance the scale of the two terms}.  
We can see that testing AUC is almost the same for both methods while the solution time of our proposed model is significantly smaller than that for $\hat F^\lambda_{hinge}$. In fact, we could not obtain solution when minimizing {\em real\_sim} within 24 hours. This difference is due to the fact that the complexity of each iteration of pairwise hinge loss optimization is superlinear in terms of $n$, while our function $F^\lambda_{nrank}$ has no dependence on $n$ at all. 

{\bf Numerical Comparison vs. LDA and ADAM}
To support our observation further, we present comparison of the linear classifiers obtained by our proposed method to those obtained by Linear Discriminant Analysis  (LDA) which is a well-known method to produce linear classifiers under the Gaussian assumption. We observe that the accuracy obtained by LDA classifiers is comparable with the other two but is significantly worse for data sets like {\em svm1} and {\em gisette}. 
In the attempt to reduce the dependence of  the complexity of optimizing $ \hat F^\lambda_{log}(w)$ and $\hat F^\lambda_{hinge}(w)$ on $n$, we  also applied popular version of stochastic gradient descent,  Adam \cite{Ba_Adam}
to the regularized logistic regression  $ \hat F^\lambda_{log}(w)$. Note that there is no reason to apply  stochastic gradient descent methods to  $F_{n01}(w)$, since the dependence on $n$ is removed from per-iteration complexity.
 
 The parameters for ADAM we chosen as recommended in \cite{Ba_Adam}, i.e. fixed step size $\alpha=0.001$, $\beta_1=0.9$, $\beta_2=0.999$ and $\epsilon=10^{-8}$. However, the results were sensitive to the choice of batch size and number of epochs.  In our experiment, after some hand tuning, we chose the number of epochs to be $500$, the same as the maximum number of iterations for deterministic method and batch size is chosen to be $200$ for most data sets. 
 The results are summarized in Table \ref{app_t_2}. For {\em colon}, batch size is chosen to be $10$ and for {\em vowel, sonar}, batch size is chosen to be $20$, since the numbers of samples for these data sets are small. In {\em poker, letter, segment, ijcnn1, w8a}, ADAM achieved comparable performance as $F_{log}$, but in other cases, it did not achieve the same accuracy as L-BFGS. 
 LDA results have been obtained by employing the {\em scikit-learn } Python package and are clearly inferior to minimizing either $F_{n01}(w)$ or $F_{log}(w)$. Both LDA and ADAM are too slow for large scale data sets such as {\em rcv1} and {\em realsim} because of large dimension $d$, hence we we did not report results on these two sets.

\begin{table}[tph]    {}
\centering
 \caption{Comparison of L-BFGS applied to $F_{n01}(w)$ and $F_{log}(w)$ vs. ADAM applied to   $F_{log}(w)$  and  LDA.}
   \vskip 0.1in
  \label{app_t_2}
   \vskip 0.05in
\begin{center}
\begin{small}
 \begin{tabular}{c|cc|c|c}  
 % \abovespace\belowspace
\multirow{2}{*}{Data} 
&{\textbf{$\pmb{F_{n01}(w)}$ Minimization}}&{\textbf{$\pmb{F_{log}(w)}$ Minimization}}&{\textbf{LDA}}&{\textbf{ADAM}}\\
&{Accuracy$\pm$ std}&{Accuracy $\pm$ std}&{Accuracy $\pm$ std}&{Accuracy $\pm$ std}\\
\hline
%\abovespace
fourclass & 0.7564$\pm$ 0.0323& 0.7602$\pm$ 0.0293&0.7572$\pm$0.0314&0.7378$\pm$0.0444 \\ 
svm1 & 0.9455$\pm$ 0.0092 &0.9306$\pm$ 0.0131&0.8972$\pm$0.0159&0.8417$\pm$0.0153\\ 
diabetes & 0.7667$\pm$ 0.0371&0.7680$\pm$ 0.0397&0.7703$\pm$0.0366&0.7333$\pm$0.3860\\ 
%shuttle&0.9383$\pm$ 0.0031  &0.9431$\pm$ 0.0023&0.9109$\pm$0.0027&0.8331$\pm$ 0.0051\\ 
vowel &0.9619$\pm$ 0.0207 & 0.9652$\pm$ 0.0176& 0.9600$\pm$0.0224 &0.9276$\pm$0.0238\\ 
magic &0.7665$\pm$ 0.0091 &0.7897$\pm$ 0.0087 &0.7841$\pm$0.0093&0.6590$\pm$0.0087\\  
poker& 0.9795$\pm$ 0.0017& 0.9795$\pm$ 0.0017&0.9795$\pm$0.0017&0.9795$\pm$0.0017\\ 
letter & 0.9710$\pm$ 0.0030  &0.9824$\pm$ 0.0019&0.9711$\pm$0.0029&0.9709$\pm$0.0034\\ 
segment &0.9845$\pm$ 0.0292& 0.9978$\pm$ 0.0022&0.9617$\pm$0.0331&0.9968$\pm$0.0032\\ 
svm3  &0.8208$\pm$ 0.0245& 0.7929$\pm$ 0.0209 &0.8238$\pm$0.0259&0.7619$\pm$0.0192\\ 
ijcnn1 &0.9054$\pm$ 0.0026 &0.9142$\pm$ 0.0025&0.9081$\pm$0.0029&0.9024$\pm$0.0023\\  
german & 0.7553$\pm$ 0.0252&0.7648$\pm$ 0.0320&0.7675$\pm$0.0275&0.7355$\pm$ 0.0308\\ 
satimage  &0.9064$\pm$ 0.0064  &0.9068$\pm$ 0.0060&0.9061$\pm$0.0065&0.8761$\pm$0.0660\\ 
sonar& 0.7573$\pm$ 0.0610 &0.7549$\pm$ 0.0761&0.7622$\pm$0.0499&0.6768$\pm$0.0703\\ 
a9a &0.8376$\pm$ 0.0043& 0.8472$\pm$ 0.0041&0.8452$\pm$0.0038&0.8066$\pm$0.0061\\ 
w8a &0.9807$\pm$ 0.0013 & 0.9842$\pm$ 0.0012&0.9839$\pm$0.0012&0.9703$\pm$0.0018\\ 
mnist &0.9819$\pm$ 0.0008&0.9877$\pm$ 0.0005&0.9778$\pm$0.0013&0.9722$\pm$0.0032\\ 
colon &0.7833$\pm$ 0.1191 & 0.7167$\pm$ 0.1221&0.8875$\pm$0.0985&0.7375$\pm$0.1332\\ 
gisette &0.9753$\pm$ 0.0035 &0.9714 $\pm$ 0.0043 &0.5875$\pm$0.0207&0.9338$\pm$ 0.0257\\
covtype& 0.5502$\pm$ 0.0134  &0.7562$\pm$0.0010& 0.7553$\pm$0.0009&0.6304$\pm$0.0142
%rcv1&?$\pm$&&?$\pm$&\\ 
\end{tabular}
\end{small}
\end{center}
\vskip -0.1in
\end{table}

{\bf Accuracy of the new approximations} We further illustrate the comparison of $F_{n01}$ and $\hat F_{log}$ by plotting these functions next to the function they are meant to approximate, which is  the empirical training error $\hat F_{01}$. In Figure \ref{fig:01plots} we show several examples of such comparisons. We have selected a segment of different $w$'s from the starting point  to the stopping point of the algorithm. We have generated $100$ equally spaced points along this path.  Note that the start and the end are different for  $F_{n01}$ and logistic regression, however, what we are trying to illustrate here is the quality of approximation these functions provide with  respect to the true empirical error in the area of interest to the optimization algorithm. For each example, on the left side we plot $F_{n01}(w)$ in red and the empirical error in blue. On the right side we plot logistic regression in red and empirical error in blue, however, due to different scaling on the functions we had to separate their plots. We see that $F_{n01}$ overall provides a better approximation of the empirical error than logistic regression. In all cases aside from {\em svm3} $F_{n01}$ behaves as a smoothed version of the empirical error. Logistic regression, on the other hand does not seem to approximate empirical error function  at all in many cases,  it only successfully predicts the area where the  minimizers of $\hat F_{01}$ lie.   Moreover, in the case of {\em colon}, which is the data set with only 62 data points, the accuracy achieved from $F_{n01}$ much better than that from logistic regression and we see that $F_{n01}$ is a very close approximation of $\hat F_{01}$. %More plots can be found in the appendix. 
%% Figure 2 3
% \begin{figure}[tph]
% \begin{subfigure}
%   \centering
%   \includegraphics[height=0.18\textwidth,width=0.48\linewidth]{Minhan/images/path_plot_zeroOne/colon_1.eps}
%   \label{fig:sfig1}
% \end{subfigure}%
% \begin{subfigure}
%   \centering
%   \includegraphics[height=0.18\textwidth,width=0.48\linewidth]{Minhan/images/path_plot_zeroOne/satimage_1.eps}
%   \label{fig:sfig2}
% \end{subfigure} 
% \begin{subfigure}
%   \centering
%   \includegraphics[height=0.18\textwidth,width=0.48\linewidth]{Minhan/images/path_plot_zeroOne/poker_1.eps}
%   \label{fig:sfig2}
% \end{subfigure} %
% \begin{subfigure}
%   \centering
%   \includegraphics[height=0.18\textwidth,width=0.48\linewidth]{Minhan/images/path_plot_zeroOne/svm3_1.eps}
%   \label{fig:sfig2}
% \end{subfigure} %
% \caption{Approximating empirical loss by $F_{n01}$ and $\hat F_{log}$.}
% \label{fig:01plots}
% \end{figure}
In Figure \ref{fig:AUCplots} we illustrate how $F_{nrank}$ and pairwise hinge loss approximate the empirical ranking loss function $\hat F_{rank}$. We observe that again $F_{nrank}$ provides  good approximations to $\hat F_{rank}$ while
$\hat F_{hinge}$ is only consistently good at approximating the minimizers.

\begin{figure}[tph]
\begin{subfigure}{0.47\textwidth}
  \centering
  \includegraphics[height=0.6\linewidth,width=\linewidth]{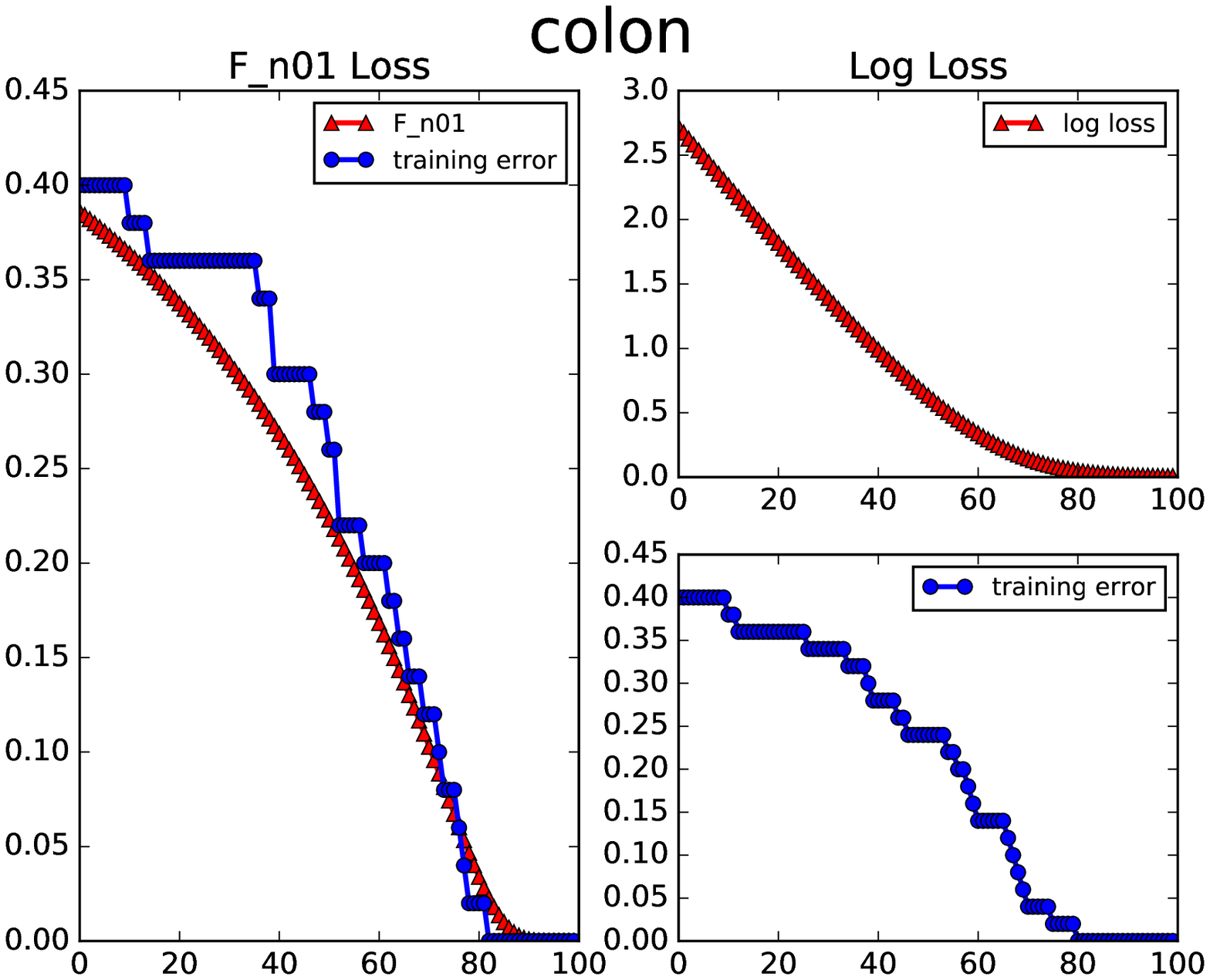}
  \label{fig:sfig1}
\end{subfigure}%
\hfill
\begin{subfigure}{0.47\textwidth}
  \centering
  \includegraphics[height=0.6\linewidth,width=\linewidth]{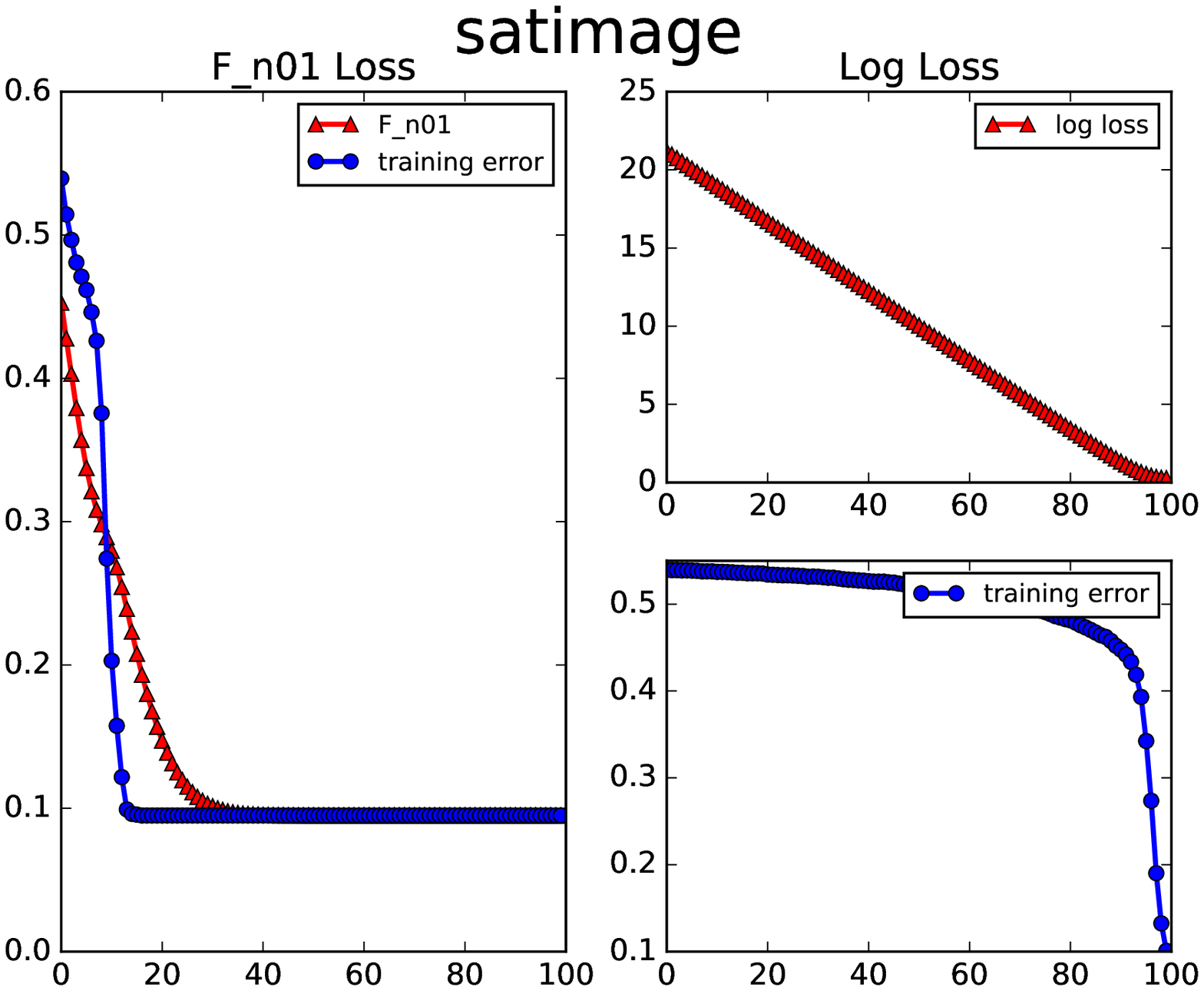}
  \label{fig:sfig2}
\end{subfigure} 
\medskip
\begin{subfigure}{0.47\textwidth}
  \centering
  \includegraphics[height=0.6\linewidth,width=\linewidth]{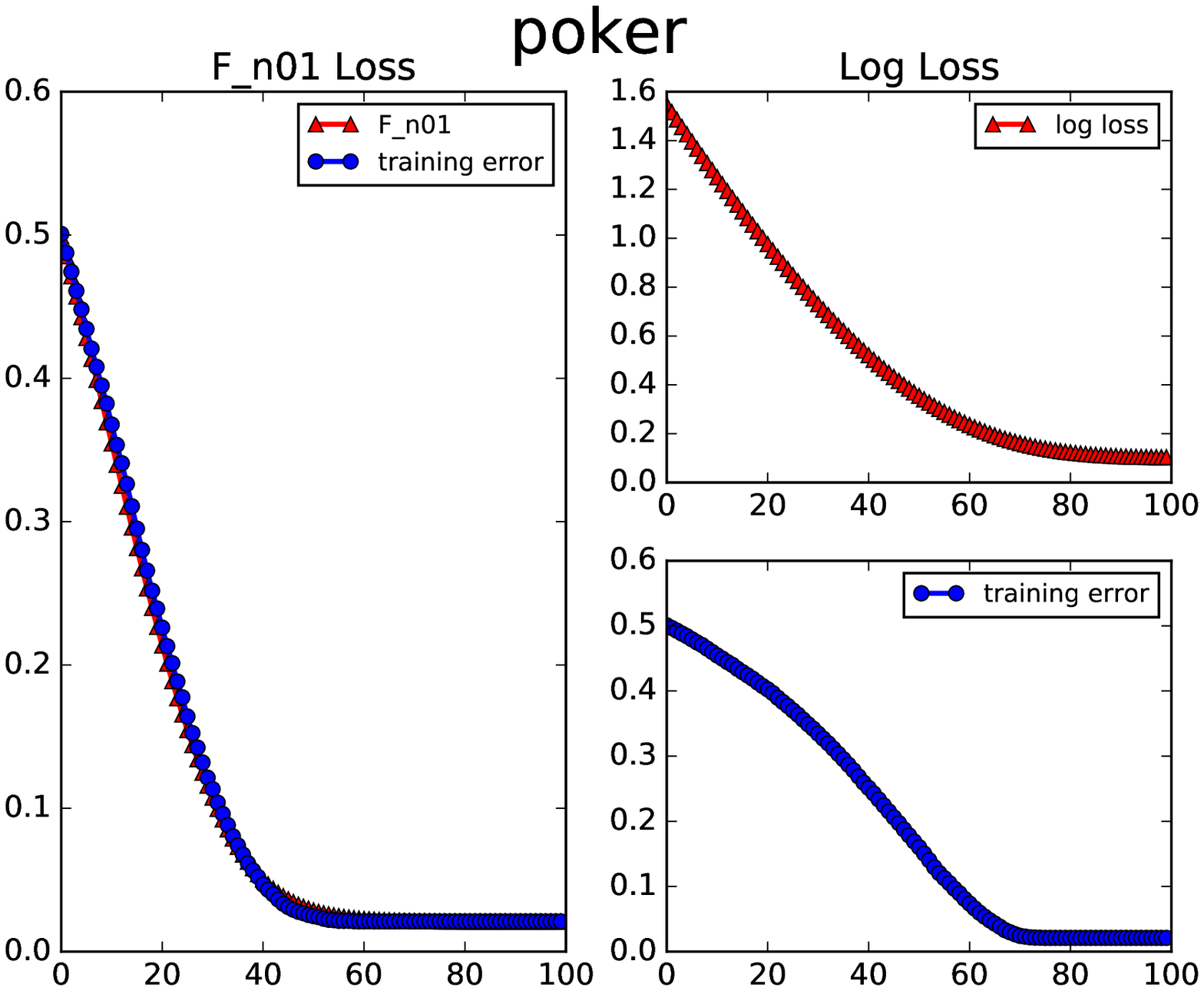}
  \label{fig:sfig2}
\end{subfigure} %
\hfill
\begin{subfigure}{0.47\textwidth}
  \centering
  \includegraphics[height=0.6\linewidth,width=\linewidth]{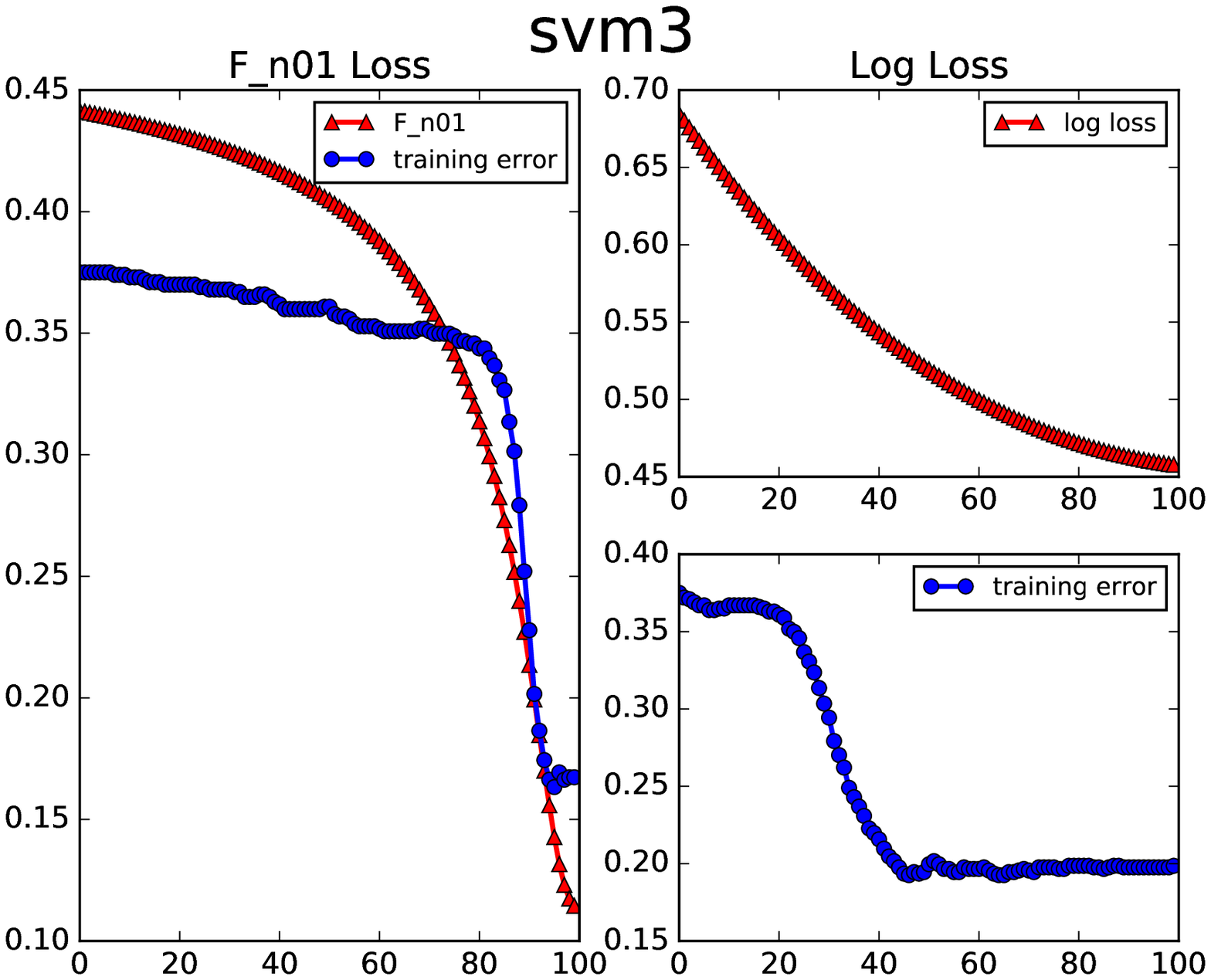}
  \label{fig:sfig2}
\end{subfigure} %
\begin{subfigure}[t]{0.47\textwidth}
  \centering
    \includegraphics[height=0.6\textwidth,width=\linewidth]{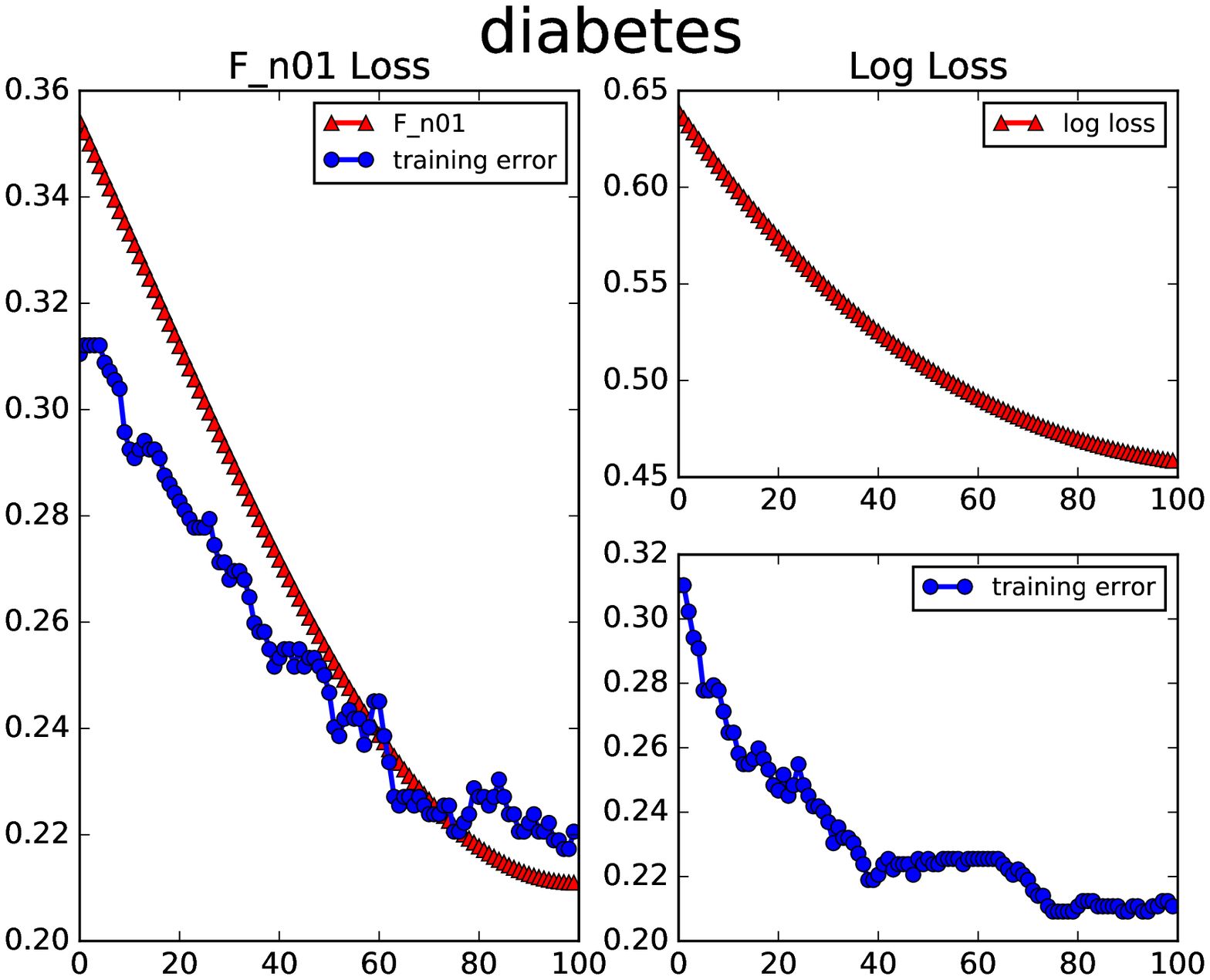}
  \label{fig:sfig1}
\end{subfigure}
\hfill
\begin{subfigure}[t]{0.47\textwidth}
  \centering
  \includegraphics[height=0.6\textwidth,width=\linewidth]{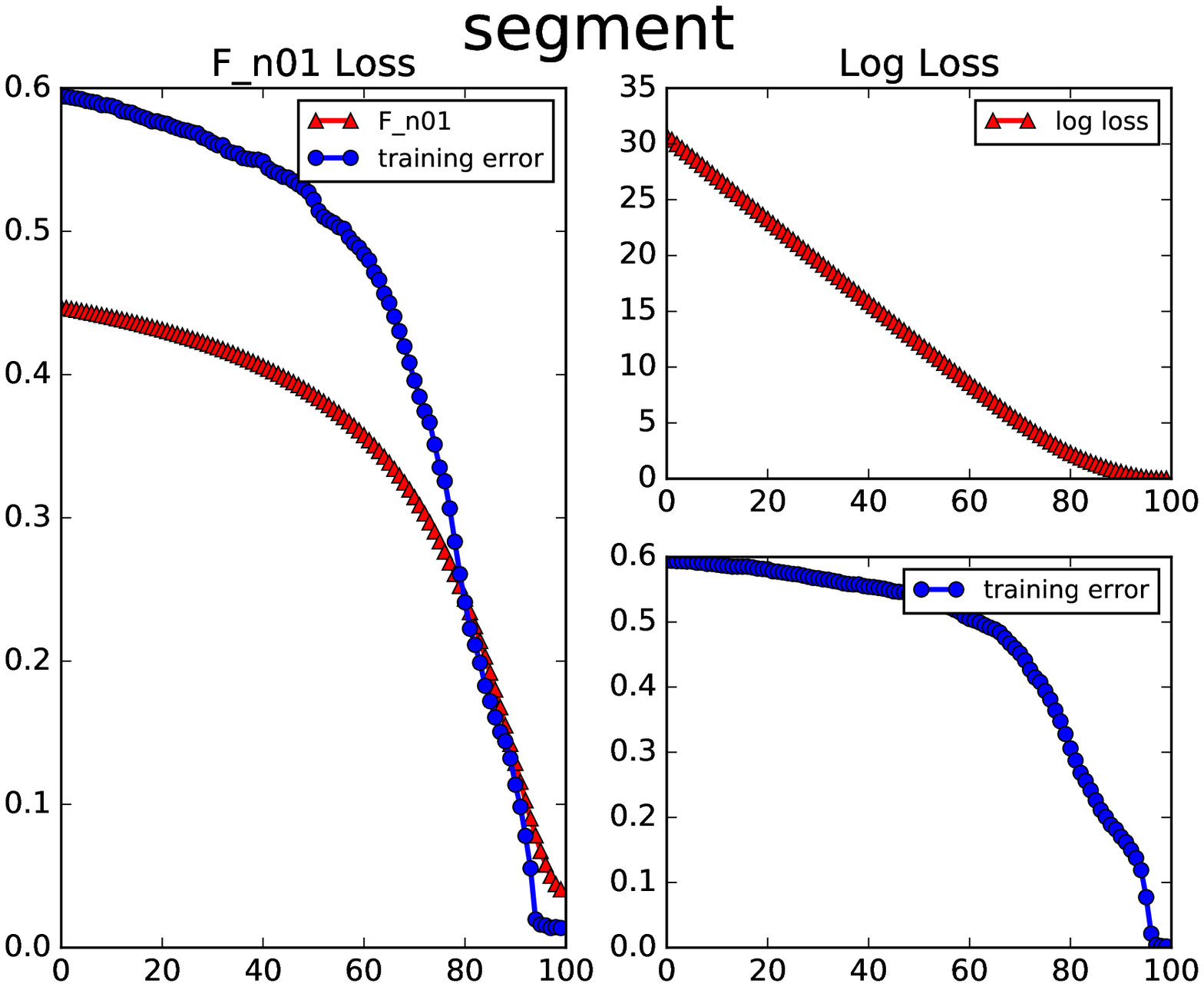}
  \label{fig:sfig2}
\end{subfigure} 
\medskip
\begin{subfigure}[t]{0.47\textwidth}
 \centering
  \includegraphics[height=0.6\textwidth,width=\linewidth]{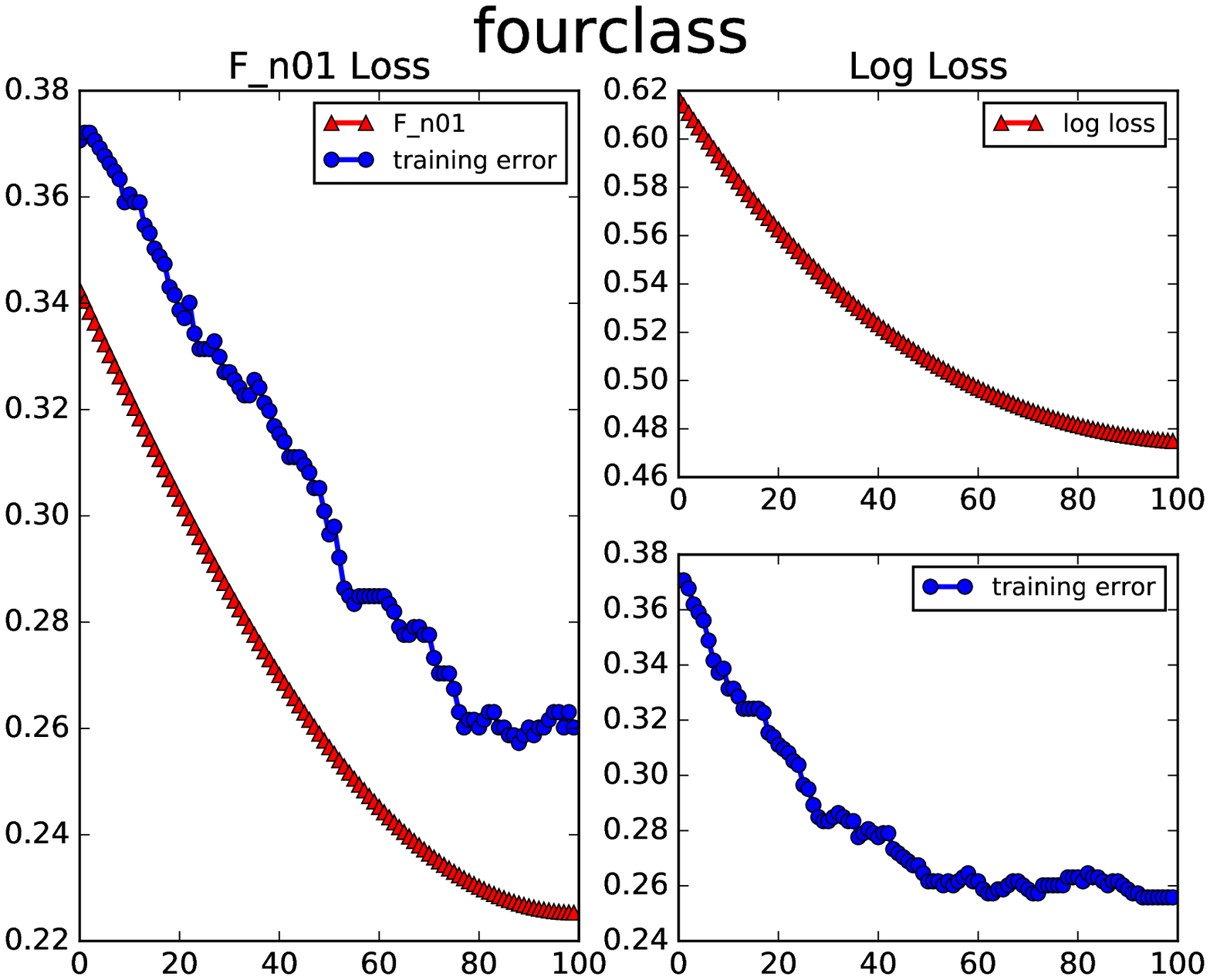}
  \label{fig:sfig2}
\end{subfigure} 
\hfill
\begin{subfigure}[t]{0.47\textwidth}
  \centering
  \includegraphics[height=0.6\textwidth,width=\linewidth]{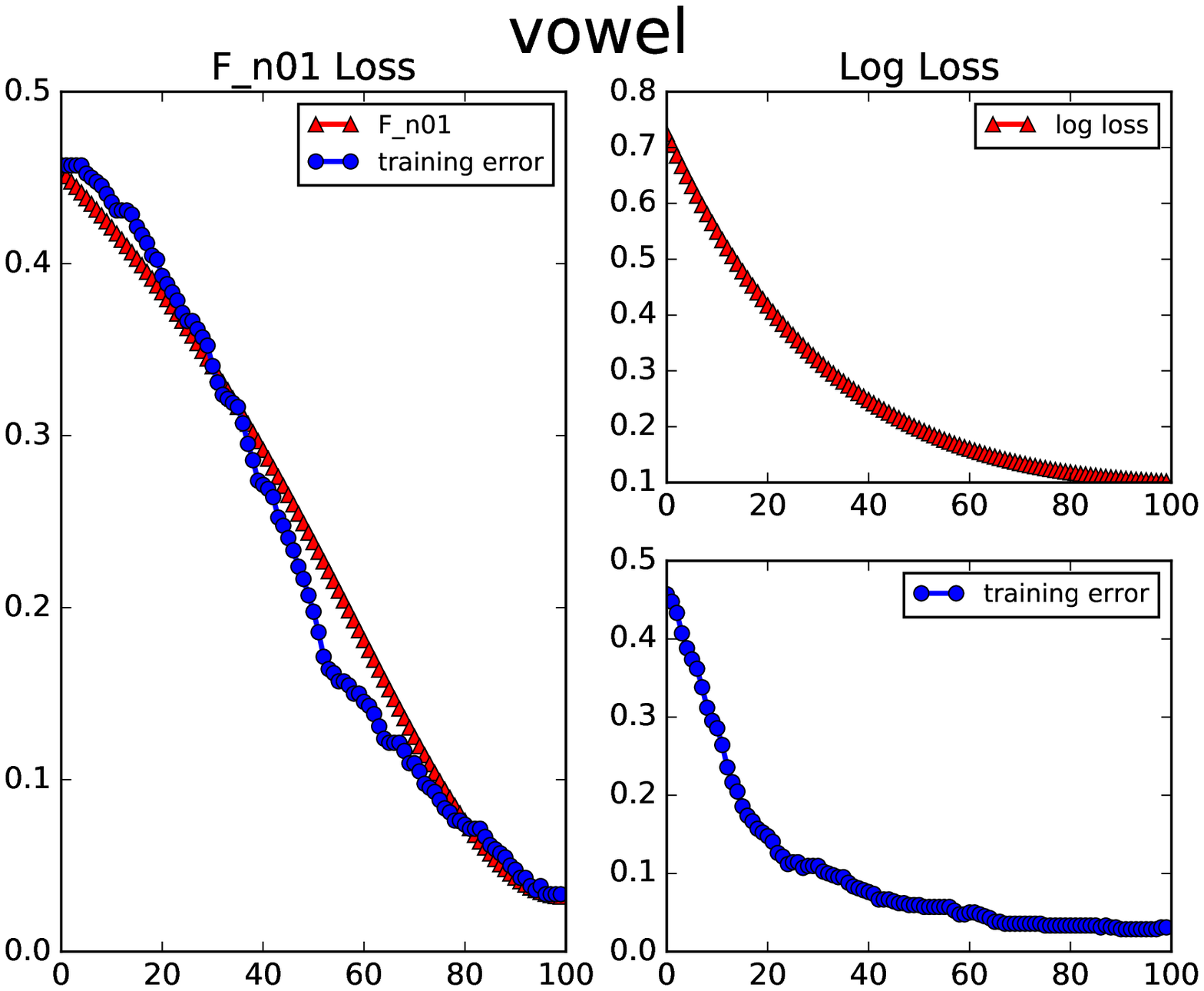}
  \label{fig:sfig2}
\end{subfigure} 
\caption{Approximating empirical loss by $F_{n01}$ and $\hat F_{log}$.}
\label{fig:01plots}
\end{figure}

\begin{figure}[tph]
\begin{subfigure}{0.47\textwidth}
  \centering
  \includegraphics[height=0.6\linewidth,width=\linewidth]{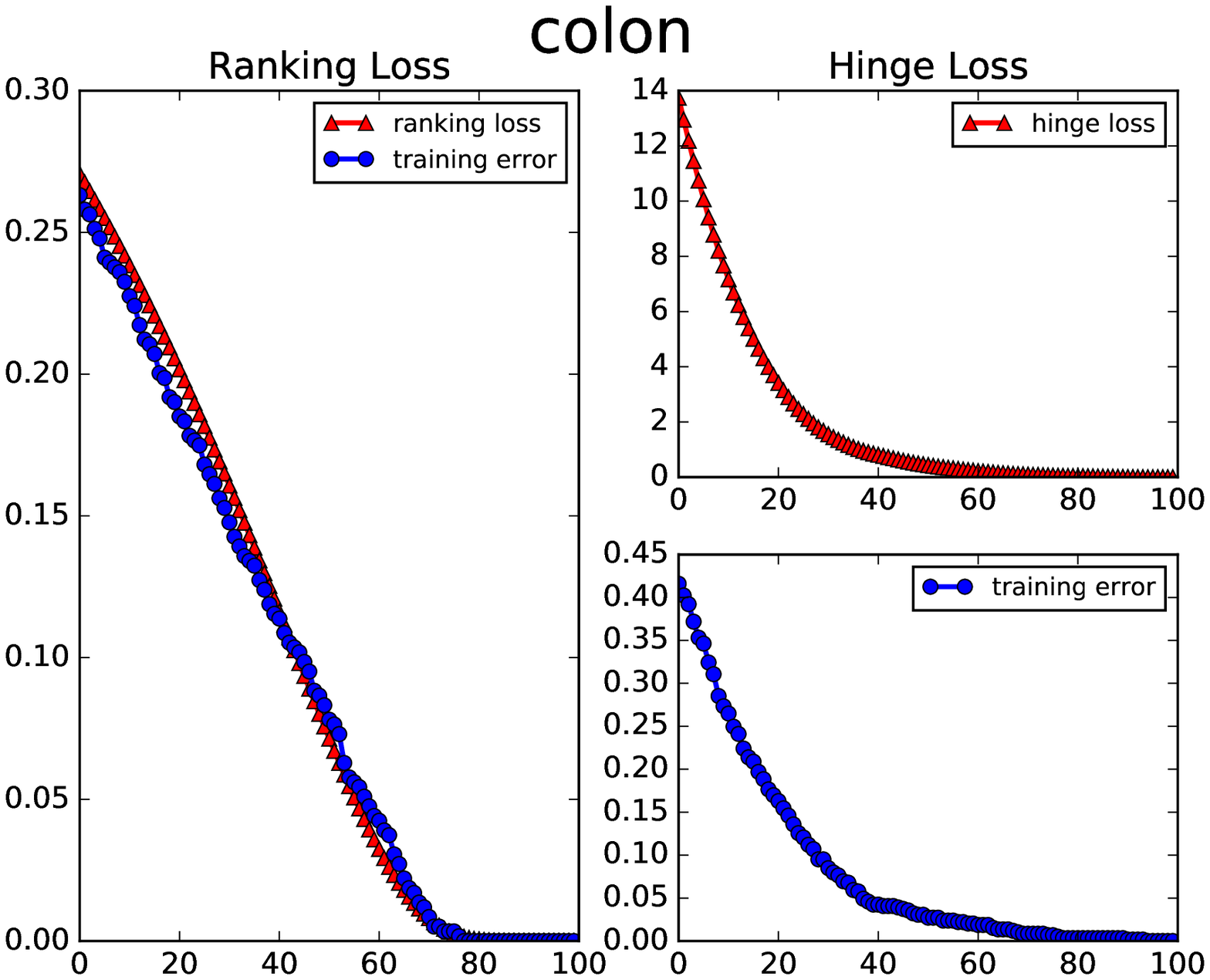}
  \label{fig:sfig1}
\end{subfigure}%
\hfill
\begin{subfigure}{0.47\textwidth}
  \centering
  \includegraphics[height=0.6\linewidth,width=\linewidth]{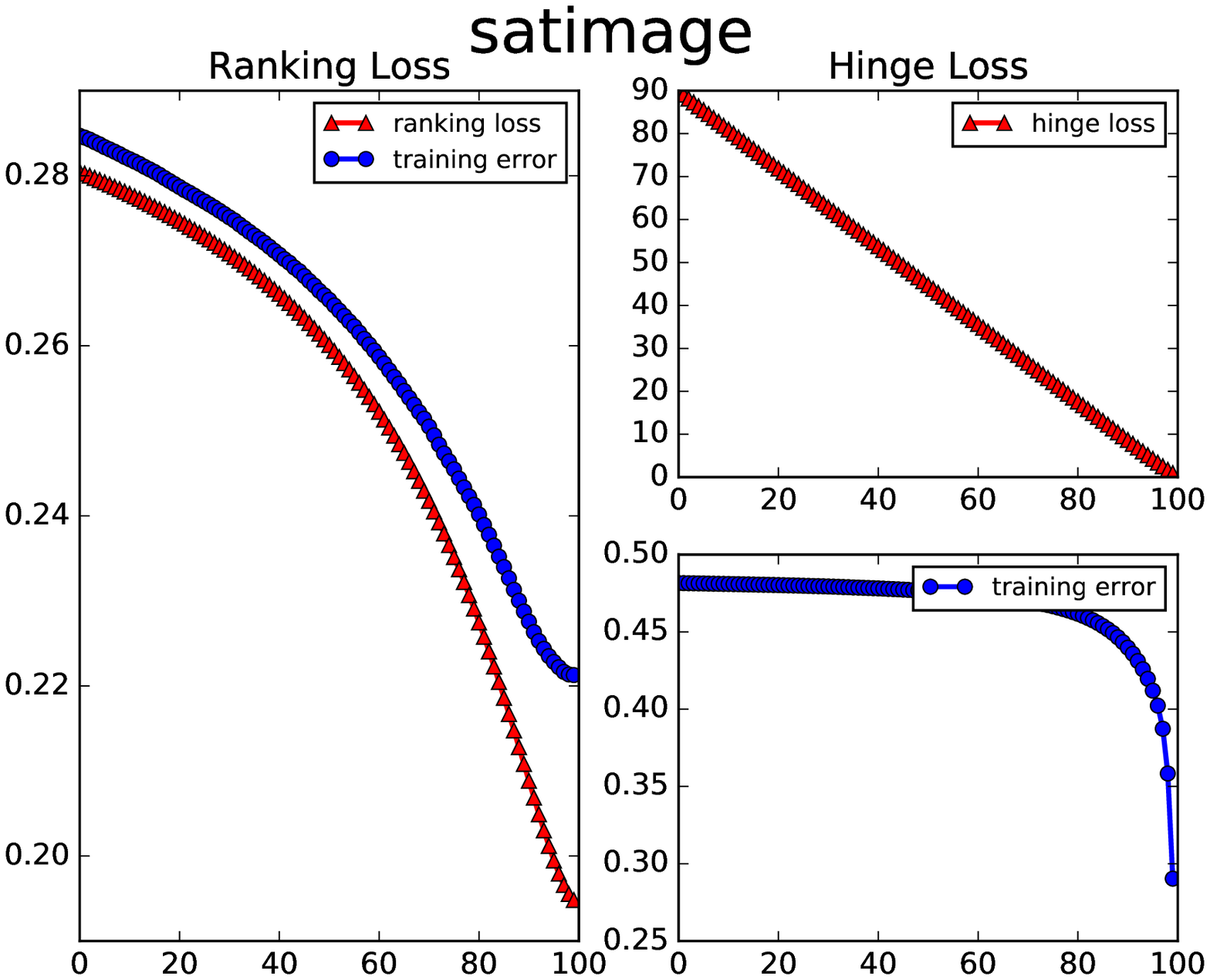}
  \label{fig:sfig2}
\end{subfigure} 
\begin{subfigure}{0.47\textwidth}
  \centering
  \includegraphics[height=0.6\linewidth,width=\linewidth]{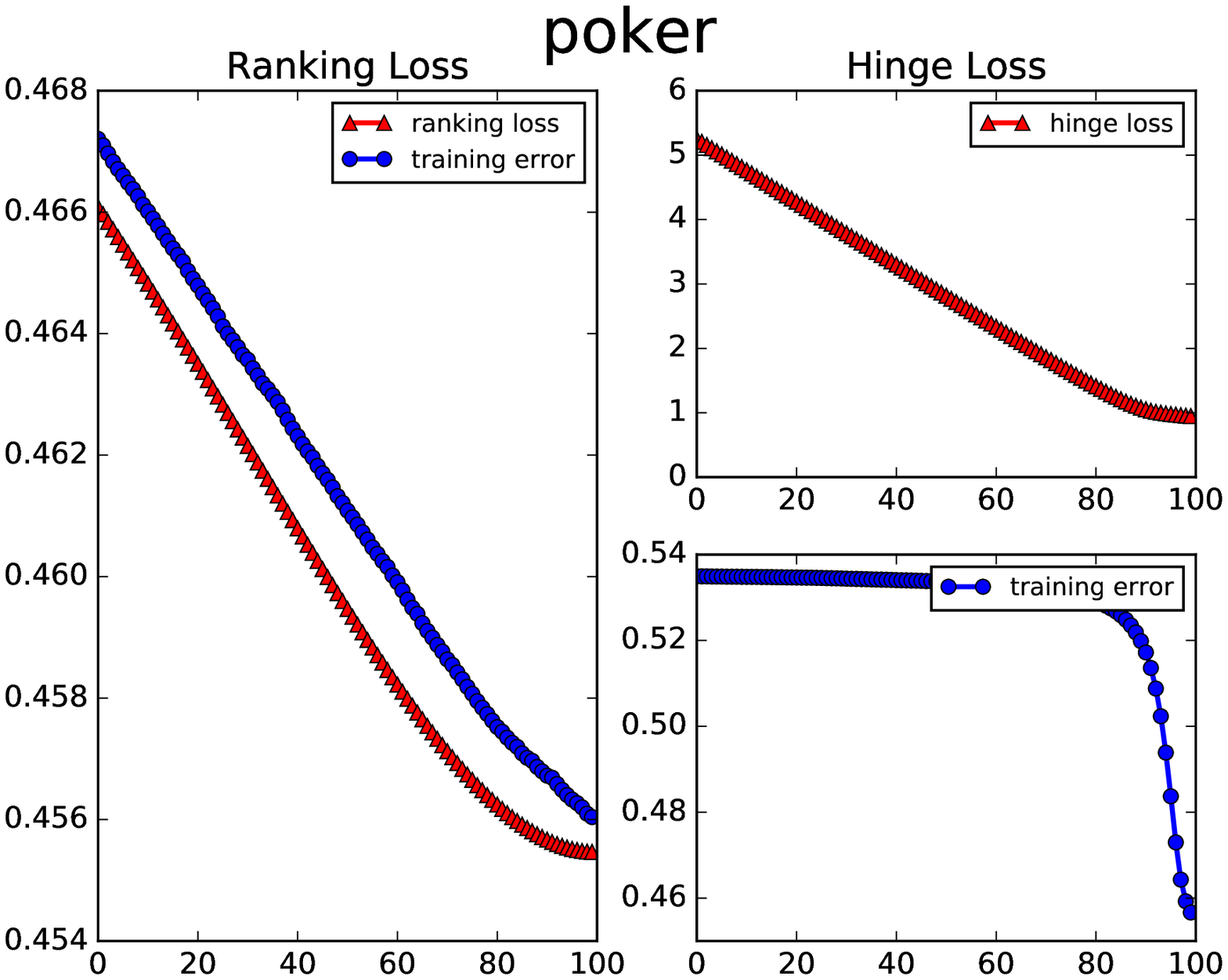}
  \label{fig:sfig2}
\end{subfigure} %
\begin{subfigure}{0.47\textwidth}
  \centering
  \includegraphics[height=0.6\linewidth,width=\linewidth]{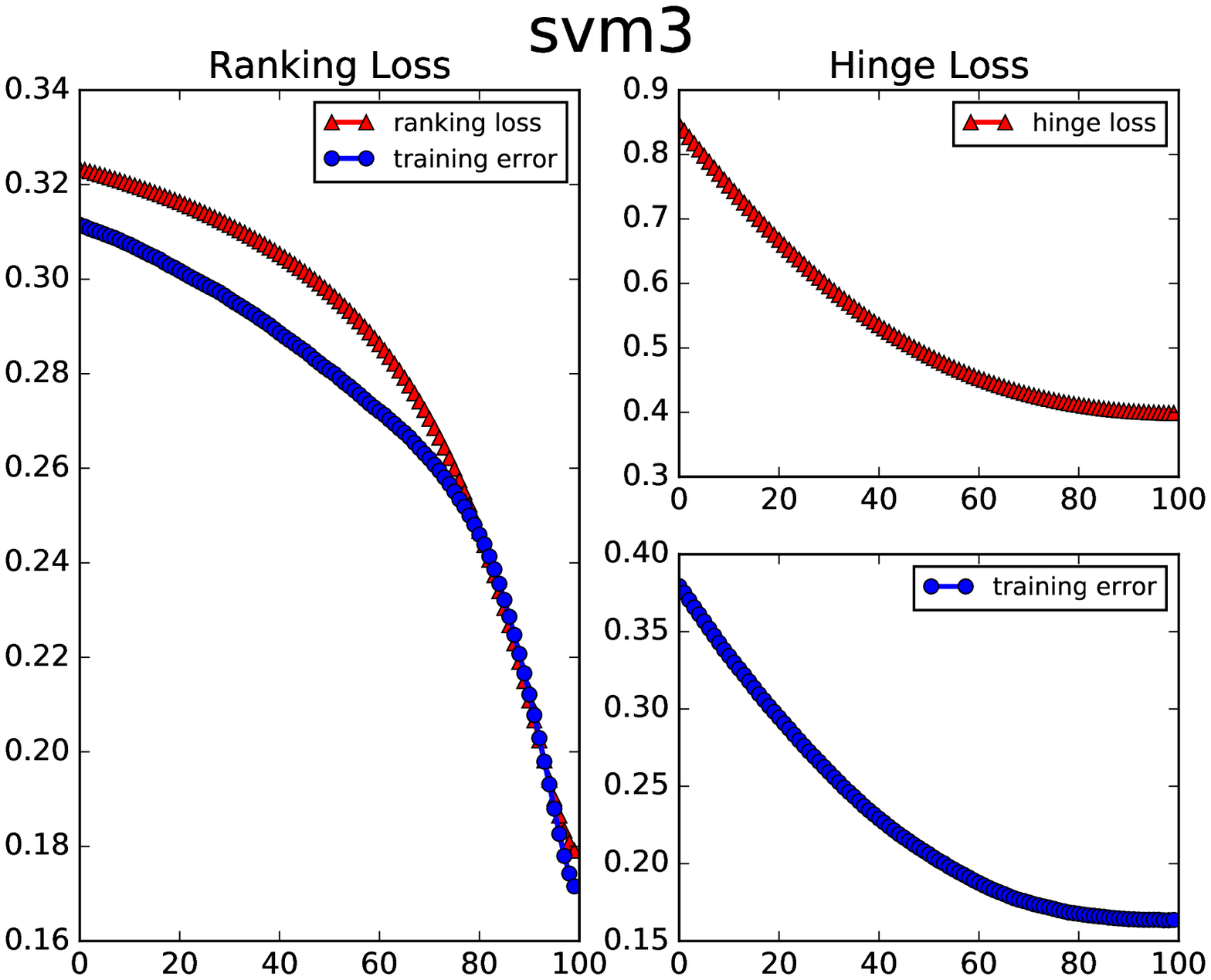}
  \label{fig:sfig2}
\end{subfigure} %
\begin{subfigure}{0.47\textwidth}
  \centering
  \includegraphics[height=0.6\textwidth,width=\linewidth]{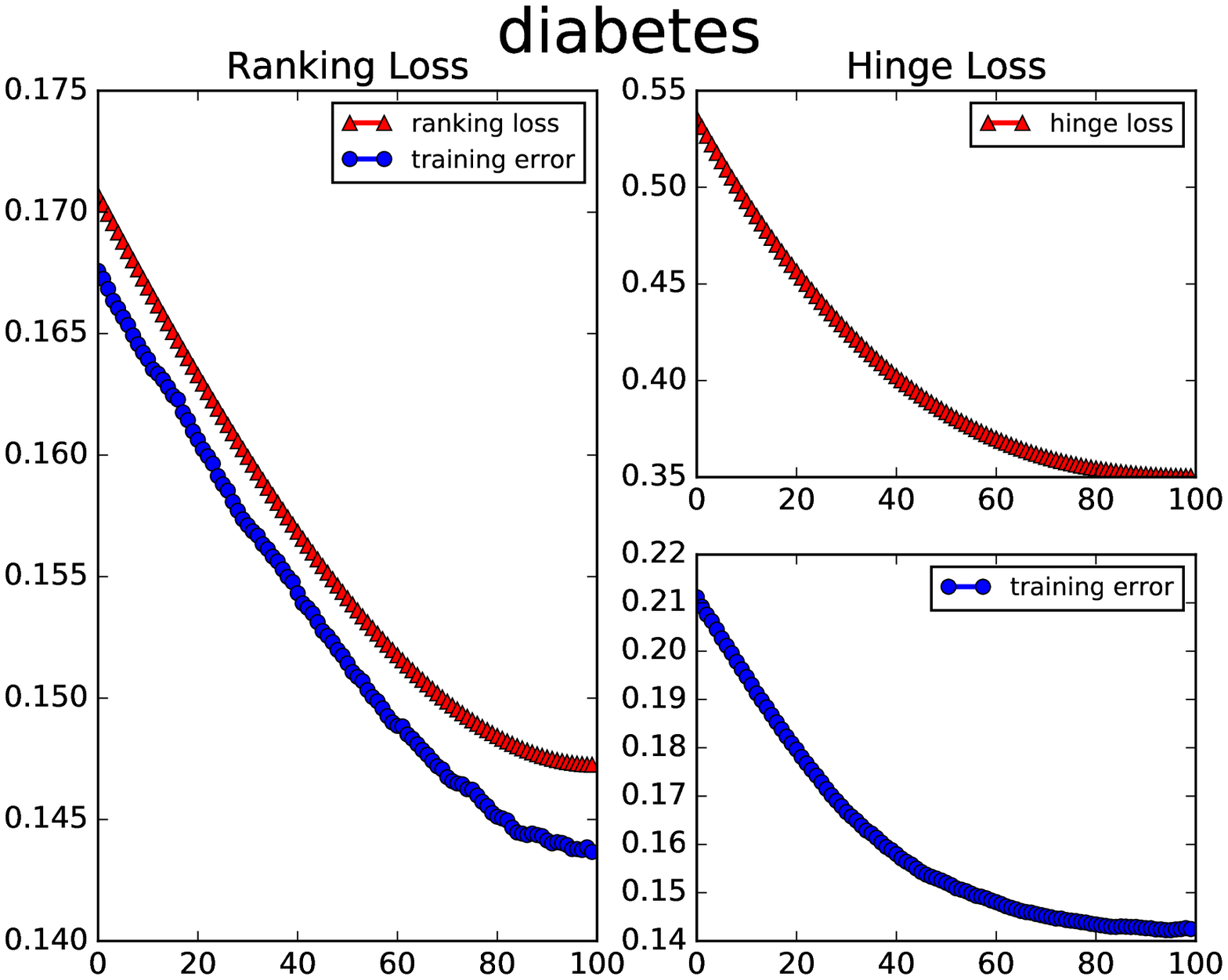}
  \label{fig:sfig1}
\end{subfigure}
\hfill
\begin{subfigure}{0.47\textwidth}
  \centering
  \includegraphics[height=0.6\textwidth,width=\linewidth]{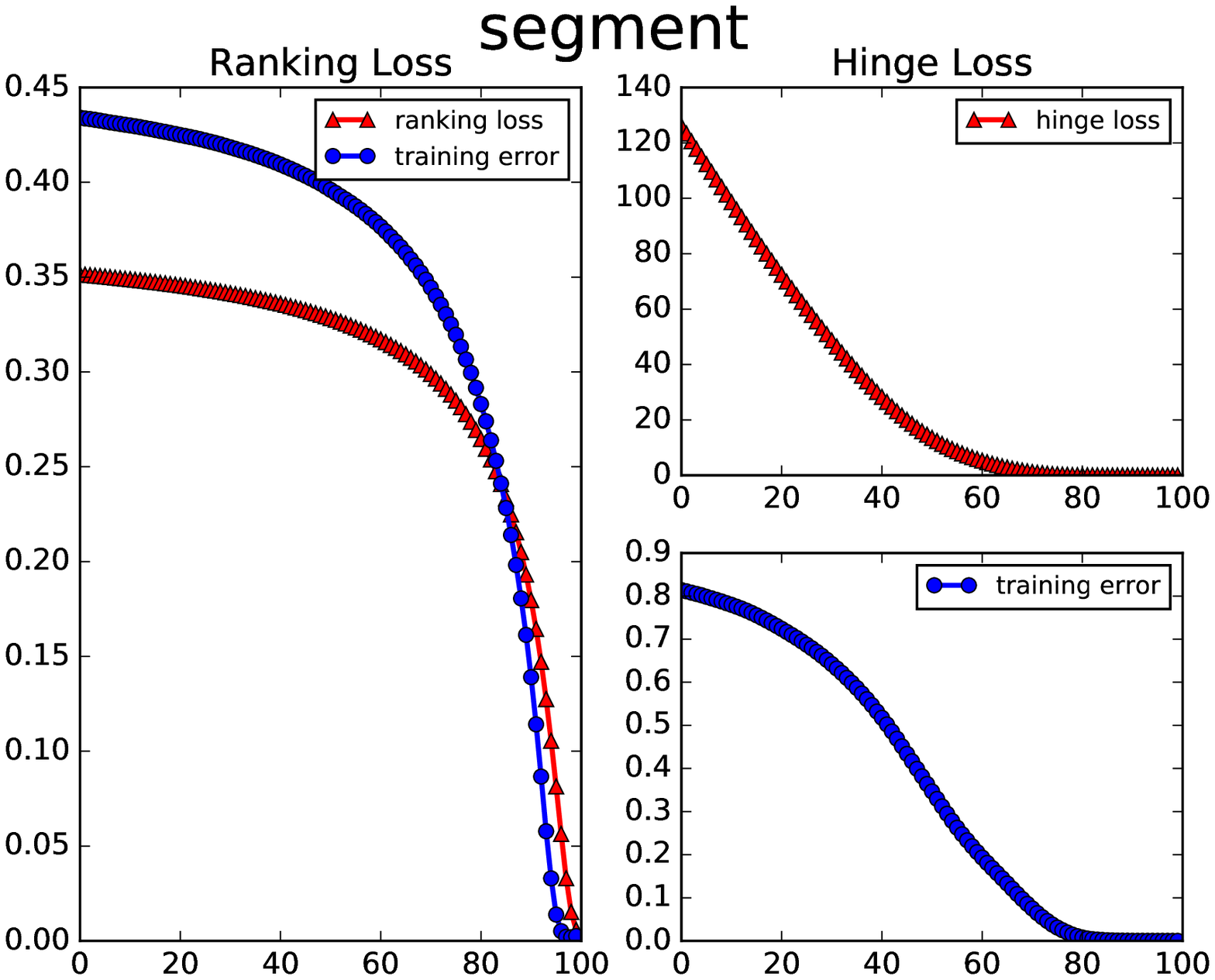}
  \label{fig:sfig2}
\end{subfigure} 
\medskip
\begin{subfigure}{0.47\textwidth}
 \centering
  \includegraphics[height=0.6\textwidth,width=\linewidth]{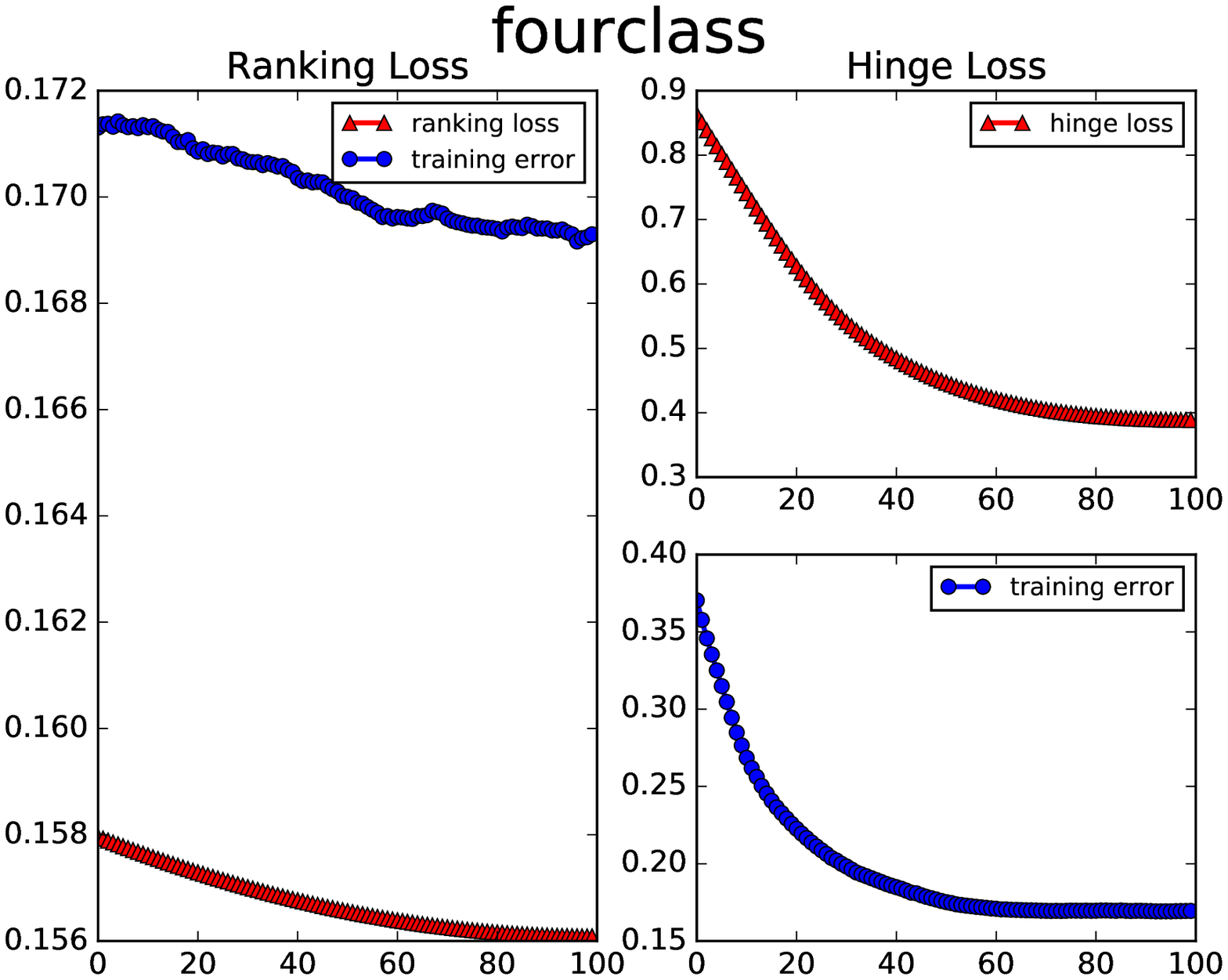}
  \label{fig:sfig1}
\end{subfigure}
\hfill
\begin{subfigure}{0.47\textwidth}
  \centering
  \includegraphics[height=0.6\textwidth,width=\linewidth]{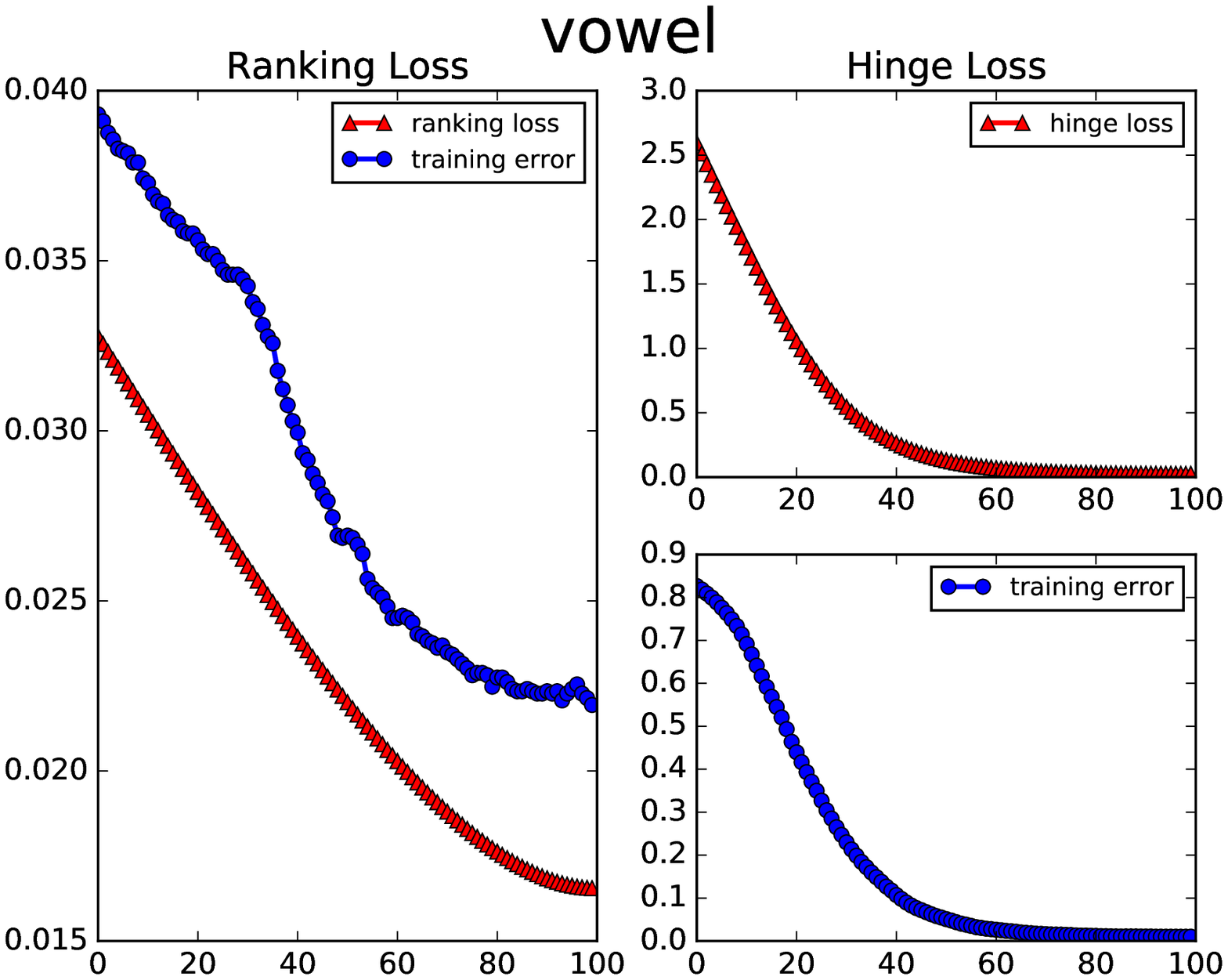}
  \label{fig:sfig2}
\end{subfigure} 
\caption{Approximating empirical ranking loss by $F_{nrank}$ and $\hat F_{hinge}$.}
\label{fig:AUCplots}
\end{figure}

\section{Conclusion}\label{sec.conclusion}
In this paper, we propose novel smooth approximation functions for the training error and ranking loss of linear predictors in binary classification, whose derivatives  are expressed using the first and second moments of the related data distribution.   We give theoretical motivation   for why and when these functions may provide good approximation. We then propose to applying an optimization algorithm to these functions to obtain linear classifiers with the test accuracy and AUC  comparable with those achieved by  state-of-the-art methods. 
The main advantage of the proposed approximations is that their evaluation  and that of  their derivatives is independent of the size of the data sets, and hence optimization algorithms applied to them can be very efficient.

\newpage
\nocite{langley00}
\bibliography{references}
\bibliographystyle{apalike}
% \nocite{langley00}
% \bibliography{references}
% \bibliographystyle{icml2017}

\end{document}